\makeatletter
\newcommand{\@disableuaihooks}{
  \def\maketitlehooka{}
  \def\maketitlehookb{}
  \def\maketitlehookc{}
  \def\maketitlehookd{}
}
\makeatother
\documentclass[accepted]{uai2026} % for initial submission
\usepackage[american]{babel}
\usepackage{natbib} % has a nice set of citation styles and commands
\usepackage{mathtools} % amsmath with fixes and additions
\usepackage{booktabs} % commands to create good-looking tables
\usepackage{tikz} % nice language for creating drawings and diagrams
\usepackage{amsmath}
\usepackage{amsfonts}
\usepackage{amsthm}
\usepackage{thmtools}
\usepackage{thm-restate}
\usepackage{subcaption}

\newtheorem{thm}{Theorem}[section] 
\newtheorem{crl}[thm]{Corollary}

\newcommand{\reals}{\mathbb{R}}
\newcommand{\loss}{\mathcal{L}}

\newcommand{\smallsum}[2]{{\textstyle%
	\sum\limits_{\scriptscriptstyle%
        #1}^{\scriptscriptstyle #2}}}
\newcommand{\smalldv}[1]{{\textstyle\dv{#1}}}
\newcommand{\smallfrac}[2]{{\textstyle\frac{\strut #1}{\strut #2}}}

\NewDocumentCommand{\dv}{ O{} m }{%
  \IfNoValueTF{#1}
    { \frac{\mathrm{d}}{\mathrm{d}#2} }
    { \frac{\mathrm{d}#1}{\mathrm{d}#2} }
}
\NewDocumentCommand{\dd}{ o m }{%
  \mathop{}\!%
  \IfNoValueTF{#1}
    {\mathrm{d} #2}
    {\mathrm{d}^{#1} #2}
}

\newcommand{\f}{f}
\newcommand{\fii}{u}
\newcommand{\fiii}{g}

\newcommand{\nonlinearity}{\phi}

\newcommand{\dimout}{K}

\newcommand{\param}{\theta}

\newcommand{\paramii}{\vartheta}

\newcommand{\paramiii}{\psi}

\newcommand{\ntkk}{\Theta}
\newcommand{\datax}{\mathcal{X}}
\newcommand{\datay}{\mathcal{Y}}

\newcommand{\nngpk}{\kappa}
\newcommand{\flin}{f_{\text{lin}}}

\newcommand{\err}{\epsilon}

\newcommand{\rnd}{\text{rnd}}

\title{On the Equivalence of Random Network Distillation,\\ Deep Ensembles, and Bayesian Inference}

\author[1]{\href{mailto:<m.a.zanger@tudelft.nl>?Subject=Your UAI 2026 paper}{Moritz A. Zanger}{}}
\author[1]{Yijun Wu}
\author[1]{Pascal R. Van der Vaart}
\author[1]{Wendelin Boehmer}
\author[1]{Matthijs T. J. Spaan}

\affil[1]{%
    Delft University of Technology\\
    Delft, 2628 XE, The Netherlands
}
  
\begin{document}
\maketitle

\begin{abstract}
Uncertainty quantification is central to safe and efficient deployments of deep learning models, yet many computationally practical methods lack lacking rigorous theoretical motivation. Random network distillation (RND) is a lightweight technique that measures novelty via prediction errors against a fixed random target. While empirically effective, it has remained unclear what uncertainties RND measures and how its estimates relate to other approaches, e.g., Bayesian inference or deep ensembles. We establish these missing theoretical connections by analyzing RND within the neural tangent kernel framework in the limit of infinite network width. Our analysis reveals two central findings in this limit: (1) The uncertainty signal from RND---its squared self-predictive error---is equivalent to the predictive variance of a deep ensemble. (2) By constructing a specific RND target function, we show that the RND error distribution can be made to mirror the centered posterior predictive distribution of Bayesian inference with wide neural networks. Based on this equivalence, we moreover devise a posterior sampling algorithm that generates i.i.d. samples from an exact Bayesian posterior predictive distribution using this modified \textit{Bayesian RND} model. Collectively, our findings provide a unified theoretical perspective that places RND within the principled frameworks of deep ensembles and Bayesian inference, and offer new avenues for efficient yet theoretically grounded uncertainty quantification methods.

\end{abstract}

\section{Introduction}\label{sec:intro}

Quantifying predictive uncertainty remains a cornerstone of reliable machine learning and underpins applications from safe robotics to efficiently exploring agents and autonomous scientific discovery. Bayesian inference is widely regarded as a theoretical gold‐standard to this end \citep{neal1996bayesian, goan2020bayesian} but its application to neural networks is typically intractable in practice, requiring approximations of simplified posteriors through variational inference \citep[VI,][]{kingmaAutoEncodingVariationalBayes2014, galDropoutBayesianApproximation2016a, blei2017variational} or complex sampling mechanisms through Markov chain Monte Carlo approaches \citep[MCMC, ][]{chen2014stochastic, liu2016stein, garriga2021exact}. Deep ensembles \citep{dietterich2000ensemble, lakshminarayananSimpleScalablePredictive2017} on the other hand maintain several independently initialized models to quantify predictive variance as uncertainty. Due to their simplicity and relative practical reliability, deep ensembles have become a widely established alternative to Bayesian approaches for uncertainty quantification in deep learning \citep{abdarReviewUncertaintyQuantification2021}.

However, both ensemble methods and approximate Bayesian methods typically incur substantial computational and memory costs, in particular for larger-scale models, motivating more efficient alternatives. RND \citep{burda2018exploration} offers one such approach: by training a \textit{predictor network} to mimic the outputs of a fixed, randomly initialized \textit{target network}, RND produces a simple novelty or uncertainty signal via the squared prediction error. Random network distillation (RND) has seen empirical success in exploration, out-of-distribution detection, and continual learning \citep{burda2018exploration, nikulin2023anti, matthews2024craftax}, yet the theoretical understanding of the nature of its uncertainty estimates remains blurry. In particular, it is unclear how---or whether---the RND error relates to the principled uncertainties produced for example by Bayesian inference or deep ensembles.

In this paper, we establish these missing theoretical connections by analyzing random network distillation in the idealized setting of infinite network width. In particular, we establish a Gaussian process (GP) interpretation of the self-predictive RND errors in the limit of infinitely wide neural networks, drawing on Neural Tangent Kernel (NTK) theory \citep{jacotNeuralTangentKernel2020, leeWideNeuralNetworks2020}. Our three main contributions are:
\begin{enumerate}%[label=(\roman*)]
  \item {\textit{Ensemble equivalence with Standard RND:}} We prove that, in the idealized infinite width limit, the squared prediction errors of standard RND coincide exactly with the variance of a deep ensemble.  
  \item {\textit{Posterior equivalence with Bayesian RND:}} By engineering the RND target function, we design a \textit{Bayesian RND} variant whose error distribution matches that of the exact Bayesian posterior predictive distribution of a neural network in the limit of infinite width.
  \item {\textit{Posterior sampling with Bayesian RND:}} Based on a multi-headed Bayesian RND model, we devise a posterior sampling algorithm that produces i.i.d. samples of the exact Bayesian posterior predictive distribution of neural networks in the limit of infinite width.
\end{enumerate}
This unifying perspective on the uncertainty estimates produced by RND, deep ensembles, and Bayesian inference provides a novel understanding and theoretical support for the empirical effectiveness of RND and suggests avenues for future research directions towards principled Bayesian inference with minimal computational overhead.
  
\section{Preliminaries}

We begin by establishing notation, defining RND formally, and briefly introducing the theoretical framework used in our analysis. In our analysis, we consider fully connected neural networks $\f(x; \param_t)$ of $L$ layers of widths $n_1, \dots, n_L = n$, parametrized by $\param_t$ at time $t$. The forward computation of such networks is defined recursively with $z_i^l(x; \param^{\leq l}_t)$ denoting the $i$-th output of layer $l$ and
\begin{align} \label{eq:forwardpass}
\begin{split}
    z_i^l(x,\param^{\leq l}_t)&=\sigma_b b_i^l + \frac{\sigma_w}{\sqrt{n_{l-1}}}\smallsum{j=1}{n_{l-1}}w_{ij}^l x_j^l(x) \\ x_j^l(x) &=\nonlinearity(z_j^{l-1}(x;\param^{\leq l-1}_t)) \,,
\end{split}
\end{align} 
where $\param^{\leq l}_t$ denotes the parameters $\{w^1,b^1, \dots, w^l, b^l\}$ up to layer $l$,  $\sigma_b$ and $\sigma_w$ denote scaling parameters of the forward computation, and $\nonlinearity:\reals \xrightarrow{} \reals$ is a Lipschitz-continuous nonlinearity. In Eq.~\eqref{eq:forwardpass}, $n_0 = d_{in}$ and $x^1(x) = x$. The output of a scalar-output neural network is then given by $\f(x;\param_t) = z^L(x;\param^{\leq L}_t)$ . We furthermore assume that parameters are initialized i.i.d. from a normal distribution $\param_0 \sim \mathcal{N}(0,I)$\footnote{Also known as NTK-parametrization, where variance scalings $\sigma_b$ and $\sigma_w$ affect both forward and gradient computations, yielding well-behaved gradients in the infinite-width limit.}). For convenience, we will sometimes overload notation to concatenate function outputs, for example indicating a set $\datax = \{x_i \in \reals^{d_{\text{in}}}\}_{i=1}^{N_D}$  and the corresponding function output as a column vector $\f(\datax;\param_t) = (\f(x_i;\param_t))_{i=1}^{N_D}$, where $\f(\datax;\param_t) \in \reals^{N_D \times \dimout}$ or matrix-valued identities $\Sigma(\datax,\datax) = ( \Sigma(x_i,x_j))_{i,j=1}^{N_D}$ , where $\Sigma(\datax,\datax) \in \reals^{N_D \times N_D}$. For conciseness our notation will furthermore use a shorthand for covariance and kernel matrices denoting $\Sigma_{\datax \datax} \equiv \Sigma (\datax,  \datax)$. In the following we briefly review methods pertinent to this work. 

\paragraph{Random network distillation.} \label{par:rnd_definition}
Random network distillation \citep{burda2018exploration} is an uncertainty quantification technique that employs two neural networks of identical architecture: A fixed, randomly initialized target network $\fiii(x; \paramiii_0): \reals^{d_{\text{in}}} \to \reals^{K}$ , and a \emph{predictor network} $\fii(x; \paramii_t)$, where parameters $\paramii_t$ are subject to optimization via gradient descent. The predictor is trained to minimize the expected squared difference to the target network's output on a set of data points $\datax = \{x_i \in \reals^{d_{\text{in}}}\}_{i=1}^{N_D}$ 
\begin{align}\label{eq:rnd_loss}
    \loss_{\rnd}(\paramii_t) = \smallfrac{1}{2} \| \fii(\datax; \paramii_t) - \fiii(\datax; \paramiii_0) \|^2_2 \,.
\end{align}
It is common to design RND with a multi headed architecture with output dimension $K$ and individual output heads $\{ \fii_i(x; \paramii_t)\}_{i=1}^\dimout$, and $\{\fiii_i(x; \paramiii_0) \}_{i=1}^\dimout$, where the sum of squared prediction errors $\err_i(x; \paramii_t, \paramiii_0) = \fii_i(x; \paramii_t) - \fiii_i(x; \paramiii_0)$  at a test point $x$ serves as an uncertainty signal
\begin{align}\label{eq:rnd_error}
\err^2(x; \paramii_t, \paramiii_0) = \smallfrac{1}{K} \smallsum{i=1}{K} \bigl( \fii_i(x; \paramii_t) - \fiii_i(x; \paramiii_0) \bigr)^2\,.
\end{align}

\paragraph{Gaussian processes.} \label{par:framework} In our analysis, we will frequently use GPs to model distributions over random functions: A univariate GP \citep{rasmussen2006gp} defines a distribution over functions $\f^{0} \sim \mathcal{GP}(\mu^{0}, \Sigma^{0})$  characterized by a mean function $\mu^{0}:\reals^{d_{\text{in}}}\xrightarrow{}\reals$ and a covariance (kernel) function $\Sigma^{0}: \reals^{d_{\text{in}}}\times\reals^{d_{\text{in}}} \xrightarrow{} \reals$ such that $f_0(\datax_T)$ follows a multivariate Gaussian distribution $f_0(\datax_T) \sim \mathcal{N}(\mu^{0}(\datax_T), \Sigma^{0}(\datax_T, \datax_T))$  for any finite set of evaluation points $\datax_T = \{x_i^{\text{Test}}\}_{i=1}^{N_T}$. We can condition a prior GP $\mathcal{N}(\mu^{0}(\datax_T), \Sigma^{0}(\datax_T, \datax_T))$ on training data $\datax = \{x_i\}_{i=1}^{N_D}$ and labels $\datay = \{y_i\}_{i=1}^{N_D}$  to obtain a posterior GP whose \emph{posterior predictive distribution} is Gaussian with mean and covariance given by
\begin{align}\label{eq:standard_GP}
\begin{split}
    \mu(\datax_T) &= \mu^{0}(\datax_T) + \Sigma^{0}_{\datax_T \datax} ( \Sigma^{0}_{\datax \datax})^{-1} \bigl(\datay-\mu^{0}(\datax) \bigr), \\
    \Sigma_{\datax_T \datax_T} &= \Sigma^{0}_{\datax_T \datax_T} - \Sigma^{0}
    _{\datax_T \datax} ( \Sigma^{0}_{\datax \datax})^{-1} \Sigma^{0}_{\datax \datax_T} \,.
\end{split}
\end{align}
% Our theoretical analysis is situated in the infinite-width limit of neural networks. In this regime, previous work has shown that neural networks at initialization are described by a GP, known as the NNGP \citep{lee2017deep}, $\f(\datax_T; \param_0) \sim \mathcal{GP}(0, \nngpk_{\datax_T \datax_T})$ with $\kappa_{\datax_T \datax_T} = \mathbb{E}_{\param_0}[\f(\datax_T; \param_0) \f(\datax_T; \param_0)^\top]$ being the NNGP kernel function. 

\paragraph{Learning dynamics with infinite width.}

We turn to analytical tools to establish solutions to the learning dynamics of neural networks in the limit of infinite width $n \to \infty$. Within this setting, we consider the training dynamics under \emph{gradient flow}, the continuous-time limit of gradient descent $\frac{d}{dt} \param_t = -\nabla_\param \mathcal{L} (\param_t)$. Under gradient flow with a square loss $\mathcal{L}(\param_t) = \frac{1}{2}\|\f(\datax;\param_t) - \datay\|_2^2$, the evolution of the NN $\f$ is described by a differential equation in function space 
\begin{align}\label{eq:function_de}
    \smalldv{t} \f(x; \param_t) &= \nabla_\param \f(x; \param_t)^\top \smalldv{t} \param_t \nonumber
    \\ 
    &= -\nabla_\param \f(x; \param_t)^\top \nabla_\param \f(\datax; \param_t) (\f(\datax; \param_t) - \datay) \nonumber \\
    &\equiv -\ntkk_t(x, \datax) (\f(\datax; \param_t) - \datay) \,.
\end{align}  
The above learning dynamics are governed by a gradient similarity function, called the \textit{neural tangent kernel} \citep[NTK,][]{jacotNeuralTangentKernel2020}, $\ntkk_t(x, x') = \nabla_\param \f(x; \param_t)^\top \nabla_\param \f(x'; \param_t)$. While this inner product is dynamic and therefore intractable in general, the limit of infinite network width yields a remarkable simplification: 1.) due to large number effects, the inner product kernel $\ntkk_0(x,x')$ at initialization is deterministic despite the random initialization of $\param_0$; 2.) $\ntkk_t(x,x')$ remains constant throughout $t$ under gradient flow \citep{jacotNeuralTangentKernel2020, leeWideNeuralNetworks2020}. In particular, this means $\lim_{n \to \infty} \ntkk_0(x, x') = \lim_{n \to \infty} \ntkk_t(x, x') \equiv \ntkk(x, x')$ and converts Eq.~\ref{eq:function_de} into a linear ordinary differential equation, which can be solved analytically. It can be shown that, under mild conditions, $\f(x;\param_t)$ converges to the kernel regression solution \citep[see ][and Appendix~\ref{thm:converged_f}]{jacotNeuralTangentKernel2020}
\begin{align}
    \f(x;\param_\infty) &= 
    \f(x;\param_0) - \ntkk_{x \datax} \ntkk_{\datax \datax}^{-1} \bigl( \datay - \f(\datax; \param_0) \bigr) \,,
\end{align}

Moreover, \citet{lee2017deep} show that both $\f(x;\param_0)$ and $\f(x;\param_\infty)$ are indeed GPs described by the neural network Gaussian process \citep[NNGP,][]{lee2017deep} $\f(x;\param_0) \sim \mathcal{GP}(0, \kappa_{x x'})$ and the converged GP defined in Theorem \ref{thm:converged_f_dist}. 
\begin{restatable}{thm}{rndwidenndistribution}\citep{leeWideNeuralNetworks2020}(Distribution of post-convergence neural network functions) \label{thm:converged_f_dist}
    Let $\f(\datax_T;\param_\infty)$ be a NN as defined in Eq.\eqref{eq:forwardpass}, and let $\datax_T$ be testpoints. For random initializations $\param_0\sim\mathcal{N}(0,I)$, and in the limit $n \to \infty$, $\f(\datax_T;\param_\infty)$ distributes as a Gaussian with mean and covariance given by
    \begin{align*}
        \mathbb{E}[\f(\datax_T, \param_\infty)] &= \ntkk_{\datax_T\datax}\ntkk_{\datax \datax}^{-1}\datay \,, \\
        \Sigma^\f_{\datax_T \datax_T}(\param_\infty) &= \nngpk_{\datax_T \datax_T} + \ntkk_{\datax_T \datax} \ntkk_{\datax \datax}^{-1} \nngpk_{\datax \datax} \ntkk_{\datax \datax}^{-1}\ntkk_{\datax \datax_T} 
        \\ & \quad - \bigl( \ntkk_{\datax_T \datax}\ntkk_{\datax \datax}^{-1}\nngpk_{\datax \datax_T} + \text{h.c.} \bigr) \,,
    \end{align*}
    where h.c. is the Hermitian conjugate of the preceding term.
\end{restatable}
See also Appendix~\ref{proofconverged_f_dist} or \citet{leeWideNeuralNetworks2020}. Note that the GP described in Theorem~\ref{thm:converged_f_dist} represents the law by which an infinite ensemble of infinitely wide neural networks from i.i.d. initializations distributes after training on $(\datax, \datay)$, but---as is---permits no Bayesian posterior interpretation, which is of the canonical form described in Eq.~\ref{eq:standard_GP}. 

\section{Equivalence of Random Network Distillation \& Deep Ensembles}\label{sec_ch3:rnd_de_equivalence}

We proceed to characterize formally the relationship between the error signals as measured by random network distillation and the predictive variance of deep neural network ensembles. Before treating multivariate output dimensions in section~\ref{sec_ch3:multiheaded_rnd}, we first consider scalar function outputs for simplicity, i.e. $\f, \fii, \fiii: \reals^{d_{\text{in}}} \xrightarrow{} \reals$ with $K=1$. This setup involves training a predictor $\fii(x;\paramii_t)$ to match a fixed random target function $\fiii(x;\paramiii_0)$. Intuitively, the expected errors ought to vanish for training points in $\datax$ and remain non-zero elsewhere, inheriting the randomness and generalization behaviors of the functions $\fii$ and $\fiii$. Owing to the linear training dynamics in the NTK regime, the dynamics of the error evolution $\smalldv{t} \err(x; \paramii_t, \paramiii_0)$ become akin to those outlined in Eq.~\eqref{eq:function_de} as
\begin{align}\label{eq:rnd_error_de}
    \smalldv{t} \err(x; \paramii_t, \paramiii_0) &= \nabla_\param \fii(x; \paramii_t)^\top \smalldv{t} \paramii_t \nonumber
    \\ 
    &= - \nabla_\paramii \fii(x; \paramii_t)^\top \nabla_\paramii \mathcal{L}_{\text{rnd}}(\paramii_t) 
    \\
    & = -\ntkk_t(x, \datax) \err(x; \paramii_t, \paramiii_0) \,. \nonumber
\end{align}  
We then draw on the results of Theorem~\ref{thm:converged_f_dist} to provide a probabilistic description of the self-predictive errors $\err(x;\paramii_\infty,\paramiii_0)$ of a converged RND model in the limit of infinite network width.
\begin{restatable}{thm}{rnderrordist}(Distribution of post-convergence RND errors)\label{thm:rnd_error_dist}
    Under NTK parametrization, let $\fii(x;\paramii_\infty)$ be a converged prediction network in $t \to \infty$, with data $\datax$ and fixed target network $\fiii(\datax;\paramiii_0)$. Let parameters $\paramii_0, \paramiii_0$ be drawn i.i.d. $\paramii_0, \paramiii_0 \sim \mathcal{N}(0,I)$, with the resulting NNGP $\fii(x;\paramii_0) \sim \mathcal{GP}(0,\nngpk^\fii(x,x'))$ and $\fiii(x;\paramiii_0) \sim \mathcal{GP}(0,\nngpk^\fiii(x,x'))$. The post-convergence RND error $\err(\datax_T;\paramii_\infty,\paramiii_0)$ is Gaussian with zero mean and covariance
    \begin{align*}
        & \mathbb{E}[\err(\datax_T, \paramii_\infty,\paramiii_0)] = 0 \,, \\
        & \Sigma^\err_{\datax_T \datax_T}(\paramii_\infty,\paramiii_0) \! = \!  \nngpk^{\err}_{\datax_T \datax_T} \!+ \ntkk_{\datax_T \datax} \ntkk_{\datax \datax}^{-1} \nngpk^\err_{\datax \datax} \ntkk_{\datax \datax}^{-1}\ntkk_{\datax \datax_T} \\ 
        &\qquad\qquad\qquad\qquad - \bigl( \ntkk_{\datax_T \datax}\ntkk_{\datax \datax}^{-1}\nngpk^\err_{\datax \datax_T} \!+ \text{h.c.} \bigr) \,,
    \end{align*} 
    where $\nngpk^{\err}_{xx'} = \nngpk^\fii_{xx'} + \nngpk^\fiii_{xx'}$ is the covariance kernel of initialization errors $\err(x; \paramii_0, \paramiii_0) = \fii(x;\paramii_0) - \fiii(x;\paramiii_0)$.
\end{restatable} 
\textbf{\textit{Proof sketch.}} The error function $\fii(x; \paramii_\infty) - \fiii(x; \paramiii_0)$ is a sum of the random post-convergence function $\fii(x;\paramii_\infty)$ and the fixed random target function $\fiii(x;\paramiii_0)$. The latter $\fiii(x; \paramiii_0)$ is known to follow the NNGP. By the linearity of NTK learning dynamics, the online function $\fii(x; \paramii_\infty)$ is an affine transformation of its initialization $\fii(x; \paramii_0)$, which itself follows the NNGP. Moreover, this affine transformation is independent of $\fiii$ or $\paramiii_0$, such that the error $\err(x;\paramii_\infty,\paramiii_0)$ is a sum of two independent GPs and therefore a GP itself. The resulting GP has zero-mean and covariance with an altered prior NNGP kernel $\nngpk^\err(x,x')$ composed of the online prior kernel $\nngpk^\fii_{xx'}$ and the target prior kernel $\nngpk^\fiii_{xx'}$. See also Appendix~\ref{proofrnd_error_dist}. 

\begin{restatable}{crl}{rnderrorexp}(Equivalence in expectation between RND errors and ensemble variance)\label{thm:rnd_error}
    Under the conditions of Theorem~\ref{thm:rnd_error_dist}, let $\err(x;\paramii_\infty,\paramiii_0)$ be the error function of a converged RND network with data $\datax$. Moreover, for a regression problem on $\datax$ for some labels $\datay$, let $\mathbb{V}[\f(x;\param_\infty)]$ denote the variance of converged NN functions random initializations. Furthermore, suppose an architectural equivalence between $\f$, $\fii$, and $\fiii$ and i.i.d. parameter initialization $\param_0, \paramii_0, \paramiii_0 \sim \mathcal{N}(0,I)$. The expected norm of the RND error $\err^2(x;\paramii_\infty,\paramiii_0)$ then coincides with the ensemble variance
    \begin{align}
        \mathbb{E}_{\paramii_0, \paramiii_0}\bigl[\err^2 (x;\paramii_\infty,\paramiii_0) \bigr] = \mathbb{V}_{\param_0}[\f(x;\param_\infty)]
    \end{align} 
\end{restatable}

\textbf{\textit{Proof sketch.}} Corollary~\ref{thm:rnd_error} follows straighforwardly from Theorem~\ref{thm:rnd_error_dist} by using $\nngpk^\fii(x,x') = \nngpk^\fiii(x,x')$. Taking the trace of the covariance matrix and dividing by $2$, we recover the predictive ensemble variance $\mathbb{V}_{\param_0}[ \f(x;\param_\infty)]$. 

Theorem~\ref{thm:rnd_error_dist} and Corollary~\ref{thm:rnd_error} formally show that, for an architectural equivalence between ensemble, predictor and target network, the expected RND errors directly quantify the predictive variance of the corresponding infinite ensemble model described by Theorem~\ref{thm:converged_f_dist}. To the best of our knowledge, it is the first formal analysis of random network distillation in the NTK regime and reveals a first theoretical motivation for the popular algorithm: in the idealized infinite-width setting, \textit{expected RND errors exactly quantify the variance of deep ensembles for any input $x$}. 

\subsection{Multi-Headed Random Network Distillation} \label{sec_ch3:multiheaded_rnd}
The analysis thus far has considered the \textit{average} behavior of scalar network outputs for simplicity. While insightful in its own right, this setting does not reflect most common practical implementations of random network distillation and instead, if taken literally, would imply an ensemble of random network distillation models. To connect with common practical implementations that typically use multi-headed architectures for enhanced reliability and efficiency, we now seek to incorporate the probabilistic relation between different function outputs $\f_i(x;\param_t)$ and $\f_j(x';\param_t)$ of a NN with shared hidden layers in the infinite-width limit. The result below identifies this relationship simply as a statistical independence between the different \textit{random} network outputs $\f_i(x;\param_t)$ and $\f_j(x';\param_t)$ for any time $t$ during gradient flow optimization. 

\begin{restatable}{prp}{rndindependenceprp}\label{thm:nngpind}(Independence of NN functions) 
Under NTK parametrization and in the limit $n \to \infty$, the random functions $f_i(x;\param_t)$ of a NN with $K$ output dimensions and shared hidden layers are mutually independent with covariance
\begin{align*}
    \Sigma_{x x'}^{ij}(\param_t) &= \mathbb{E}[\f_i(x;\param_t) \f_j(x';\param_t)] = 
    \begin{cases}
        \Sigma^\f_{xx'}(\param_t) & \,i=j \,, \\
    0 & \, i\neq j \,,
    \end{cases}
\end{align*}
on the interval $t \in [0, \infty)$.
\end{restatable}

\textbf{\textit{Proof sketch.}} The property follows from known results that state the independence between output dimensions of the NNGP kernel $\nngpk$ and the NTK $\ntkk$ \citep{arora2019exact, lee2017deep, jacotNeuralTangentKernel2020}. For both kernels, the proof proceeds by induction, where the independence property between output dimensions is propagated layer-wise. The induction start is equal for both kernels, where first layer outputs, as well as gradients are linear transformations of the Gaussian first-layer weights. Both the NNGP and NTK permit a recursive formulation, through which the independence property can be propagated layer-wise, constituting the induction step. Combined with the learning dynamics of wide NNs, we can conclude that the individual function outputs of a multi-headed NN, too, are statistically independent for any time $t$ on the interval $[0,\infty)$. See Appendix~\ref{proofnngpind} or \citet{lee2017deep} and \citet{jacotNeuralTangentKernel2020}.

Notably, this decoupling holds despite the shared hidden layers and is an artifact of the learning dynamics exhibited in the infinite width limit and the NTK regime. In the absence of feature learning, output functions become statistically independent despite sharing a network body. By virtue of this independence property, a translation of the earlier obtained single-function results on RND error distributions (Theorem~\ref{thm:rnd_error_dist} and Corollary~\ref{thm:rnd_error}) to the multi-headed setting is straightforward. Our next result thus establishes an equivalence between the errors of the multi-headed RND algorithm, a widely used architecture in practice, and the variance of a finite-sized deep ensemble. 

\begin{restatable}{thm}{rndequivalence}\label{thm:rndequivalence}(Distributional equivalence between multi-headed RND and finite deep ensembles)
    Under the conditions of Theorem~\ref{thm:rnd_error_dist}, let $\fii_i(x;\paramii_\infty), \fiii_i(x;\paramiii_0)$ be the $i$-th output of predictor and target networks respectively with $K$ output dimensions. Denote their sample mean RND error $\bar{\err}^2(x; \paramii_\infty,\paramiii_0) = \frac{1}{K} \sum_{i=1}^{K} \err_i^2(x;\paramii_\infty,\paramiii_0)$. Moreover, let $\{ \f(x;\param^i_\infty)\}_{i=1}^{K+1}$ be an ensemble of $K+1$ NNs from i.i.d. initial draws $\param_0$. Denote its sample variance $\bar{\sigma}_f^2 (x; \param_\infty^{i\dots K+1}) = \frac{1}{K}\sum_{i=1}^{K+1} ( \f(x;\param^i_\infty) - \frac{1}{K+1}\sum_{j=1}^{K+1} \f(x;\param^j_\infty) )^2$. The sample mean RND error and sample ensemble variance distribute to the same law
    \begin{align}
        \smallfrac{1}{2} \bar{\err}^2(x; \paramii_\infty,\paramiii_0) \overset{D}{=} \bar{\sigma}_f^2 (x; \param_\infty^{i\dots K+1}) \,,
    \end{align}
    where $\overset{D}{=}$ indicates an equality in distribution, namely by a scaled Chi-squared distribution $\bar{\sigma}_f^2(x; \param_\infty^{i\dots K+1}) \sim \frac{\Sigma^\f_{x x}(\param_\infty)}{K} \chi^2(K)$ with scale $\Sigma^\f_{x x}(\param_\infty)$ given by the analytical variance as given in Theorem~\ref{thm:converged_f_dist}.
\end{restatable} 
\paragraph{Proof sketch. } By Proposition~\ref{thm:nngpind}, the function heads $\{\fii_i(x;\paramii_\infty)\}_{i=1}^\dimout$ are $K$ independent predictors, each trained to match their independent targets $\fiii_i(x;\paramiii_0)$. Thus, the errors $\{\err_i(x;\paramii_\infty,\paramiii_0)\}_{i=1}^\dimout$ are i.i.d. samples from the error distribution outlined in Proposition~\ref{thm:rnd_error}. In particular, $\bar{\err}^2$ is the empirical mean of i.i.d. samples from a Gaussian which is known to be Chi-squared distributed. Similarly, we have that the ensemble $\{ \f(x;\param^i_\infty)\}_{i=1}^{K+1}$ are $K+1$ i.i.d. samples from the GP defined in Theorem~\ref{thm:converged_f_dist}, again yielding the known Chi-squared distribution for its sample variance $\bar{\sigma}_f^2 (x; \param_\infty^{i\dots K+1})$. See Appendix~\ref{proofrndequivalence}.

Theorem~\ref{thm:rndequivalence} establishes a distributional equality between the empirical error of a multi-headed RND architecture and the empirical variance of a finite ensemble of neural networks in the limit of infinite width, providing a theoretical motivation for the use of RND and its common multi-headed architecture as an uncertainty quantification technique. 

In a broader sense, we believe this analysis is insightful to many practitioners using random network distillation by establishing an intuitive link between theory and practice. Still, the NTK-based perspective applies to an inherently idealized regime and naturally opens up new avenues for investigation. Understanding the relationship between RND networks and deep ensembles at finite width, where feature learning impacts behavior, remains a critical open question beyond the scope of our current framework. Yet, intriguing possibilities also arise within the infinite-width setting itself: Could the properties of the RND target network be deliberately chosen or modified? Exploring different target initializations offers a computationally inexpensive lever to shape the uncertainty signal captured by RND. Indeed, pursuing this very direction, the next section investigates how a specific adaptation of the RND target function allows us to establish a direct correspondence not just with ensemble variance, but with the principled uncertainty quantification provided by Bayesian posterior inference.

\section{Equivalence of Random Network Distillation \& Bayesian Posteriors}\label{sec_ch3:rnd_bayes_equivalence}

Having formulated an equivalence between standard random network distillation and deep ensemble variance, we now proceed to investigate how theoretical connections to the Bayesian inference framework can be established by invoking deliberate changes to the standard random network distillation algorithm, namely by modifying the fixed target function $\fiii$. Our goal is to show that the RND error signal itself can, under specific conditions, be interpreted as a draw from a centered Bayesian posterior predictive distribution.

To this end, we briefly recall Bayesian inference with the classical Gaussian linear model. We define a regression model as $\f(x;\param) = \phi(x)^\top\param$  with a feature mapping $\phi:\reals^{d_{\text{in}}} \xrightarrow{} \reals^{d_{P}}$, and a prior distribution over the parameters $p(\param) \sim \mathcal{N}(0,\Sigma^{0})$. The prior distribution $p(\param)$ implicitly defines a GP prior $f^0(x;\param) \sim \mathcal{GP}(0, \phi(x)^\top \Sigma^{0} \phi(x'))$, with the prior kernel $K_{xx'} = \phi(x)^\top \Sigma^{0} \phi(x')$. Within this linear model\footnote{We use a noise-free regression model for ease of notation here, but extensions to the noisy case by including an observation noise term $\sigma_n^2 I$ in the kernel matrix inversions (cf. Eq.~\eqref{eq:bayesian_gp}-\eqref{eq:ntkgppost}) are straightforward.}, we look to infer a posterior distribution over functions given observations $\datax = \{x_i \in \reals^{d_{\text{in}}}\}_{i=1}^{N_D}$  and labels $\datay = \{y_i \in \reals\}_{i=1}^{N_D}$. Owing to our prior choice, the corresponding \textit{posterior predictive} distribution conditioned on $\datax,\datay$ is a GP with
\begin{align} \label{eq:bayesian_gp}
    p(\f|x,\datax,\datay) \sim \mathcal{N} ( K_{x \datax} K^{-1}_{\datax\datax}\datay, \, K_{xx} - K_{x\datax}K^{-1}_{\datax\datax} K_{\datax x} ) \,.
\end{align}
When contrasting this identity with the GP governing the distribution of converged NN functions of Theorem~\ref{thm:rnd_error_dist}, one observes a disparity in the structure of the covariance functions. While Theorem~\ref{thm:converged_f_dist} and Theorem~\ref{thm:rnd_error_dist}, too, specify GPs, they do not permit an interpretation as a Bayesian posterior predictive distribution \citep{leeWideNeuralNetworks2020} due to the presence of two (in general) distinct kernel functions, namely the NNGP kernel $\nngpk$ and the NTK $\ntkk$. However, inspection of Theorem~\eqref{thm:rnd_error_dist} and Eq.~\eqref{eq:bayesian_gp} suggests a path: if the prior kernel components within $\Sigma^\err_{xx'}(\paramii_\infty, \paramiii_0)$, namely $\nngpk_{xx'}^\err$, are be aligned with the dynamics kernel $\ntkk_{xx'}$ (i.e., if $\nngpk^\err \propto \ntkk$), then the resulting covariance structure simplifies to the desired Bayesian posterior form of 
\begin{align} \label{eq:ntkgppost}
    \f(x;\param_\infty) \sim \mathcal{N}\bigl( 
    \ntkk_{x \datax} \ntkk_{\datax \datax}^{-1} \datay, \,
    \ntkk_{x x} - \ntkk_{x \datax} \ntkk^{-1}_{\datax \datax} \ntkk_{\datax x}
    \bigr) \,. 
\end{align}
An important insight here is that Eq.~\ref{eq:ntkgppost} now is the \textit{exact Bayesian posterior predictive distribution of a neural network in the infinite width limit}, which corresponds to a kernel regression model with the NTK as a GP prior $\mathcal{GP}(0, \ntkk_{xx'})$ and conditioned on the data $(\datax,\datay)$. 

The idea of aligning the prior and dynamic kernels has been previously explored by \citet{heBayesianDeepEnsembles2020a} to construct \emph{Bayesian ensembles} where the predictive distribution of the ensemble matches the posterior predictive distribution of the \emph{NTK-GP}. We propose that a similar alignment can be achieved in the RND framework by constructing the target function $\fiii(x;\paramiii_0)$ to assume a specific form. The idea is to design a target $\tilde{\fiii}(x;\paramii_0, \paramiii_0)$ such that when a predictor $\fii(x;\paramii_0)$ is trained to match it, the resulting ``Bayesian'' error distribution $\err^b(x; \paramii_\infty, \paramii_0, \paramiii_0) = \fii(x;\paramii_\infty) - \tilde{\fiii}(x;\paramii_0,\paramiii_0)$ behaves like a draw from the posterior of a Bayesian model whose prior kernel is the NTK $\ntkk_{xx'}$ itself \footnote{The newly constructed target function $\tilde{\fiii}(x;\paramii_0,\paramiii_0)$ uses both $\paramii_0$ and $\paramiii_0$ for reasons that will become clear in the remainder of section. }. 

In the random network distillation algorithm, the prior kernel $\nngpk^{\err^b}_{xx'}$ of initialization errors $\err^b(x; \paramii_0, \paramii_0, \paramiii_0) = \fii(x;\paramii_0) - \tilde{\fiii}(x;\paramii_0, \paramiii_0)$ is given by the sum of the online prior kernel and the target prior kernel $\nngpk^{\err^b}_{xx'} = \nngpk^\fii_{xx'} + \nngpk^{\tilde{\fiii}}_{xx'}$ (cf. Theorem~\ref{thm:rnd_error_dist}), provided that $\fii$ and $\tilde{\fiii}$ follow independent GPs. To obtain an error prior kernel that aligns with the NTK such that $\nngpk^{\err^b}_{xx'} = \ntkk_{xx'}$, one may thus construct the target prior such that it satisfies $\nngpk^{\tilde{\fiii}}_{xx'} = \ntkk_{xx'} - \nngpk^{\fii}_{xx'}$. To this end, a closer inspection of the relation between the NNGP kernel $\nngpk^\fii_{xx'}$ and the NTK $\ntkk_{xx'}$ is instructive. For this purpose, we will view the online network $\fii(x;\paramii_0)$ as a random feature model with its forward computation path as described in Eq.~\ref{eq:forwardpass}. Let in this scenario $x^{L}(x)$ denote the output vector, or the post-activations, before the final linear layer and denote the last-layer parameters at initialization $t=0$ as $(w^L, b^L)$. We can write the NN output at initialization $\fii(x;\paramii_0)$ as
\begin{align} \label{eq:bayesforwardpass}
    \fii(x;\paramii_0) &= \sigma_b b^L + \frac{\sigma_w}{\sqrt{n_{L-1}}}\smallsum{i=1}{n_{L-1}}w_{i}^L x_i^L(x) \,,
\end{align}
that is, as a simple linear model of the random final post-activations $x^{L}(x)$. Viewing the function in Eq.~\eqref{eq:bayesforwardpass} as a random feature model leads to a central insight: since the last-layer weights and biases $(w^L, b^L)$ are assumed to be initialized i.i.d. from a standard normal $(w^L, b^L) \sim \mathcal{N}(0,I)$, Eq.~\eqref{eq:bayesforwardpass} describes a (random) affine transformation of a Gaussian vector \footnote{To see the correspondence in Eq.~\ref{eq:nngplastlayer}, first notice that due to the i.i.d. initialization of $(w^L, b^L)$, any cross-products (e.g., involving elements indexed with $i \neq j$) vanish in the expectation $\mathbb{E}[\fii(x;\paramii_0) \fii(x';\paramii_0)]$. The expectation thus becomes $\mathbb{E}[\fii(x;\paramii_0) \fii(x';\paramii_0)] = \mathbb{E}_{w^{\leq L}, b^{\leq L}}[ \sigma_b^2 + \frac{\sigma_w^2}{n_{L-1}}\sum_{i=1}^{n_{L-1}} x^L_i(x) x^L_i(x')]$. By linearity, the expectation on the r.h.s. can be pulled inside the sum and by symmetry we have that $\mathbb{E}_{w^{\leq L}, b^{\leq L}}[x^L_i(x) x^L_i(x')]$ is independent of $i$, s.t. $\mathbb{E}_{w^{\leq L}, b^{\leq L}}[\frac{\sigma_w^2}{n_{L-1}}\sum_{i=1}^{n_{L-1}} x^L_i(x) x^L_i(x')] = \sigma_w^2 \mathbb{E}[x^L_i(x) x^L_i(x')]$.} whose covariance in the limit $n \to \infty$ is quantified by the NNGP kernel $\nngpk^\fii_{xx'}$ given by
\begin{align} \label{eq:nngplastlayer}
    \nngpk^\fii_{xx'} = \mathbb{E}[\fii(x;\paramii_0) \fii(x';\paramii_0)] = \sigma_b^2 + \sigma_w^2 \mathbb{E}[x^L_i(x) x^L_i(x')] \,. 
\end{align}
Let us now compare this expression for the the prior kernel $\nngpk^\fii_{xx'}$ of the online network with its dynamics kernel $\ntkk_{xx'}$. In particular, we will split the dynamics kernel $\ntkk_{xx'}$ into a last-layer component $\ntkk_{xx'}^L = \nabla_{\{ w^L, b^L\}} \fii(x;\paramii_0)^\top \nabla_{\{ w^L, b^L\}} \fii(x';\paramii_0)$ and a component summarizing all preceding parameters $\ntkk_{xx'}^{\leq L-1} = \nabla_{\paramii^{\leq L-1}} \fii(x;\paramii_0)^\top \nabla_{\paramii^{\leq L-1}} \fii(x';\paramii_0)$ such that $\ntkk_{xx'} = \ntkk_{xx'}^L + \ntkk_{xx'}^{\leq L-1}$. Since $\fii(x;\paramii_0)$ is linear in the last-layer parameters $\{ w^L, b^L\}$ (cf. Eq.~\ref{eq:bayesforwardpass}), we make the crucial observation that the last-layer NTK component $\ntkk^L_{xx'}$ equals the NNGP prior kernel $\ntkk_{xx'}^L = \nngpk^\fii_{xx'}$ \footnote{To see this correspondence, notice that the last-layer gradient inner product $\nabla_{\{ w^L, b^L\}} \fii(x;\paramii_0)^\top \nabla_{\{ w^L, b^L\}} \fii(x';\paramii_0)$ reduces to the sum $\sigma_b^2 + \frac{\sigma_w^2}{n_{L-1}}\sum_{i=1}^{n_{L-1}} x^L_i(x) x^L_i(x')$, where the r.h.s. sum tends to its expectation in the limit $n_{L-1} \to \infty$ given that summands are identically distributed (as before by symmetry) and independent (which is shown more rigorously for example in Sec.~\ref{proofnngpind}).}. This property gives a clear instruction for engineering the prior kernel of the target network: by constructing $\nngpk^{\tilde{\fiii}}_{xx'}$ such that $\nngpk^{\tilde{\fiii}}_{xx'} = \ntkk_{xx'}^{\leq L-1}$ and independently from $\nngpk^{\fii}_{xx'}$, we obtain an error prior as 
\begin{align} \label{eq:bayeserrorprior}
    \nngpk^{\err^b}_{xx'} = \nngpk^{\tilde{\fiii}}_{xx'} + \nngpk^{\fii}_{xx'} = \ntkk_{xx'}^L + \ntkk_{xx'}^{\leq L-1} = \ntkk_{xx'} \,.
\end{align}
In the following, we will thus aim to construct a target function $\tilde{\fiii}(x;\paramii_0, \paramiii_0)$ with the desired property $\nngpk^{\tilde{\fiii}}_{xx'} = \ntkk_{xx'}^{\leq L-1}$, in particular by modeling $\tilde{\fiii}$ as a linear function in the feature space corresponding to gradients in earlier layers. This approach has also previously been explored by \citet{heBayesianDeepEnsembles2020a} to obtain Bayesian ensembles.

\begin{restatable}{prp}{rndntkgpinit}\label{thm:bayesrndrarget} (Bayesian RND target function)
Under the conditions of Theorem~\ref{thm:rnd_error_dist}, let $\fii(x;\paramii_0)$ and $\fiii(x;\paramiii_0)$ be neural networks of $L$ layers with parameters $\paramii_0, \paramiii_0 \sim \mathcal{N}(0,I)$ i.i.d. Moreover, let $\paramiii_0^{L} = \{ w^L, b^L\}$ denote the last-layer parameters of $\paramiii_0$ and $\paramiii_0^{\leq L-1}$ the parameters of all preceding layers. Suppose the target function $\tilde{\fiii}(x;\paramii_0, \paramiii_0)$ is given by
\begin{align*}
    \tilde{\fiii}(x;\paramii_0, \paramiii_0) = \nabla_{\paramii_0} \fii(x;\paramii_0)^\top \paramiii_0^* \,,
\end{align*}
where $\paramiii_0^*=\{\paramiii_0^{\leq L-1}, 0_{\text{dim}(\paramiii_0^{L})}\}$ is a copy of $\paramiii_0$ with its last-layer weights set to $0$. In the infinite width limit $n \to \infty$, $\tilde{\fiii}(x;\paramii_0, \paramiii_0)$ distributes by construction as $\tilde{\fiii}(x;\paramii_0, \paramiii_0) \sim \mathcal{GP}(0, \nngpk_{xx'}^{\tilde{\fiii}})$ where $\nngpk_{xx'}^{\tilde{\fiii}} = \ntkk_{xx'}^{\leq L-1}$.
\end{restatable}

\textbf{\textit{Proof sketch.}} The function $\tilde{\fiii}(x;\paramii_0, \paramiii_0)$ is by construction equivalent to a linear function with the (random) feature map $\nabla_{\paramii_0^{\leq L-1}}\fii(x;\paramii_0)$ given by the gradient of parameters in the pre-final layers and with a parameter vector $\paramiii_0^{\leq L-1}$. Conditioned on $\paramii_0$, the random function $\tilde{\fiii}(x;\paramii_0,\paramiii_0)$ is thus an affine transformation of the Gaussian vector $\paramiii_0^{\leq L-1}$ and thus a GP itself, at any width $n$. Using the central results by \citet{jacotNeuralTangentKernel2020} that $\ntkk_{0, xx'} \to \ntkk_{ xx'}$ as $n \to \infty$ and appealing to the bounded convergence theorem, the limiting distribution of the \textit{unconditioned} random function $\tilde{\fiii}(x;\paramii_0,\paramiii_0)$, too, becomes Gaussian with the deterministic covariance $\ntkk^{\leq L-1}_{xx'}$. 

While the specific form of the kernel $\ntkk_{xx'}^{\leq L-1} = \ntkk_{xx'} - \ntkk_{xx'}^L$ seems unusual as a standalone prior, it is crucially important in shaping the final error distribution. This is because with the altered ``Bayesian'' target function $\tilde{\fiii}(x;\paramii_0, \paramiii_0)$ we can shape the covariance structure of errors at initialization by satisfying Eq.~\ref{eq:bayeserrorprior}, appealing to Theorem~\eqref{thm:rnd_error_dist}. With the engineered target function $\tilde{\fiii}(x;\paramii_0, \paramiii_0)$, the learning dynamics of an RND model where the predictor network $\fii(x;\paramii_t)$ learns to mimic $\tilde{\fiii}(\datax;\paramii_0, \paramiii_0)$ can be shaped in the desired way. Our central statement is that the distribution of the error between the converged predictor $\fii(x;\paramii_\infty)$ and the target function $\tilde{\fiii}(x;\paramii_0, \paramiii_0)$ will then no longer reflect the variance of deep ensembles trained with gradient descent, but will instead directly exhibit the statistics of a Bayesian posterior predictive distribution derived from the NTK-GP prior. Theorem~\ref{thm:bayes_rnd_error_dist} formalizes this result.

\begin{restatable}{thm}{rndbayesrnderrordist}(Distribution of Bayesian RND errors)\label{thm:bayes_rnd_error_dist}
    Under the conditions of Theorem~\ref{thm:rnd_error_dist}, let $\fii(x;\paramii_\infty)$ be a converged predictor network trained on data $\datax$ with labels from the fixed target function $\tilde{\fiii}(\datax;\paramii_0, \paramiii_0)$ as defined in Proposition~\ref{thm:bayesrndrarget}. Let parameters $\paramii_0, \paramiii_0$ be drawn i.i.d. $\paramii_0, \paramiii_0 \sim \mathcal{N}(0,I)$. The convergenced Bayesian RND error $\err^b(\datax_T;\paramii_\infty,\paramii_0, \paramiii_0) = \fii(\datax_T;\paramii_\infty) - \tilde{\fiii}(\datax_T;\paramii_0, \paramiii_0)$ on a test set $\datax_T$ is Gaussian with zero mean and covariance
    \begin{align*}
        \Sigma^{\err^b}_{\datax_T \datax_T}(\paramii_\infty,\paramii_0,\paramiii_0) &= \ntkk_{\datax_T \datax_T} - \ntkk_{\datax_T \datax} \ntkk_{\datax \datax}^{-1} \ntkk_{\datax \datax_T} \,,
    \end{align*}
    and thus recovers the covariance of the exact Bayesian posterior predictive distribution of an infinitely wide neural network with the corresponding NTK $\ntkk_{xx'}$.
\end{restatable}

\textbf{\textit{Proof sketch.}} The result follows by combining Theorem~\ref{thm:rnd_error_dist} and Proposition~\ref{thm:bayesrndrarget}, provided that the GP governing the predictor initialization $\nngpk^\fii_{xx'}$ and the target function $\nngpk^{\tilde{\fiii}}_{xx'}$ are independent. Owing to the fact that the parameters $\paramii_0$ and $\paramiii_0$ are drawn independently, the independence between $\fii(x;\paramii_0)$ and  $\tilde{\fiii}(x;\paramii_0, \paramiii_0)$ is apparent by rewriting the covariance $\mathbb{E}[\fii(x;\paramii_0) \tilde{\fiii}(x;\paramii_0, \paramiii_0)]$ in terms of conditional expectations on $\paramii_0$ by the law of total expectation. Furthermore, since $\ntkk_{xx'} = \ntkk^{L}_{xx'} + \ntkk^{\leq L-1}_{xx'}$ and $\nngpk^{\tilde{\fiii}}_{xx'} = \ntkk^{\leq L-1}_{xx'}$, $\nngpk^\fii_{xx'} = \ntkk^{L}_{xx'}$, we have that $\nngpk^{\err^b}_{xx'}=\ntkk_{xx'}$. In other words, the GP kernel of initial errors aligns with the NTK of the online predictor, such that the distribution of post-convergence errors in Theorem~\ref{thm:rnd_error_dist} simplifies significantly. This same covariance function indeed also defines the posterior predictive distribution of infinitely wide neural networks as described by the GP with prior $\mathcal{GP}(0, \ntkk_{xx'})$ and conditioned on $(\datax, \datay)$.

Theorem~\ref{thm:bayes_rnd_error_dist} shows that with a specifically engineered target function, the RND error signal $\err^b(x;\paramii_\infty, \paramii_0, \paramiii_0) = \fii(x;\paramii_\infty) - \tilde{\fiii}(x;\paramii_0, \paramiii_0)$ is no longer just related to ensemble variance, but rather becomes a direct sample from the centered posterior predictive distribution of a Bayesian model whose prior kernel is the NTK itself. This novel result provides a direct bridge between RND and Bayesian inference in the limit of infinite network width, providing a useful insight: the error signal generated by this modified RND procedure is not merely a heuristic measure of distance, but is itself a random draw from the (centered) Bayesian posterior predictive distribution of an NTK-based GP. This direct distributional equivalence has immediate practical implications, for example prescribing rather straightforwardly how this Bayesian form of RND can be used for exact posterior sampling. By applying Proposition~\ref{thm:nngpind} to the multi-headed Bayesian RND architecture\footnote{In a multi-headed architecture, the Bayesian target function described in Proposition~\ref{thm:bayes_rnd_error_dist} becomes a JVP. Several common machine learning libraries (e.g., JAX \citep{jax2018github} offer dedicated algorithms to compute such JVPs efficiently.}, in contrast to obtaining samples from deep ensembles as done in Theorem~\ref{thm:rndequivalence}, we now obtain several independent samples from the centered posterior predictive distribution through $\err^b_i(x;\paramii_\infty, \paramii_0, \paramiii_0) = \fii_i(x;\paramii_\infty) - \tilde{\fiii}_i(x;\paramii_0, \paramiii_0)$. The below corollary details how this can be leveraged to conduct a posterior sampling procedure, requiring access only to a mean estimate and a single Bayesian RND model.

\begin{crl}[Posterior Sampling via Bayesian RND] \label{crl:bayes_rnd_sampling}
    Let $\mathcal{N}\bigl( \mu^b(x) \,, \,\, \Sigma^b_{xx'} \bigr)$ be the posterior predictive distribution of an infinitely wide neural network conditioned on $x$ with mean $\mu^b(x) = \ntkk_{x \datax} \ntkk^{-1}_{\datax \datax} \datay$ and covariance $\Sigma^b_{xx'} = \ntkk_{xx'} - \ntkk_{x \datax} \ntkk^{-1}_{\datax \datax} \ntkk_{\datax x'}$. Suppose $\tilde{\mu}(x;\param_\infty) \approx \mu^b(x)$ is an estimate of the mean function and let $\{\err^b_i(x;\paramii_\infty, \paramii_0, \paramiii_0)\}_{i=1}^\dimout$ be error functions of a $K$-head Bayesian RND model as defined in Theorem~\ref{thm:bayes_rnd_error_dist}. 
    
    \noindent The following procedure generates (at most $K$) independent samples from the conditional posterior predictive distribution $\mathcal{N}\bigl( \mu^b(x) \,, \,\, \Sigma^b_{xx'} \bigr)$: 
    \begin{enumerate}%[label=(\alph*)]
        \item sample $i \sim \mathcal{U}[1, K]$
        \item compute $\tilde{\mu}_i(x) = \tilde{\mu}(x;\param_\infty) + \err^b_i(x; \paramii_\infty, \paramii_0, \paramiii_0)$ 
        \item $\tilde{\mu}_i(x)$ is an i.i.d. sample from the conditional posterior predictive $\mathcal{N}\bigl( \mu^b(x) \,, \,\, \Sigma^b_{xx'} \bigr)$
    \end{enumerate}
\end{crl} 

\textbf{\textit{Proof sketch.}} The result follows directly from Theorem~\eqref{thm:bayes_rnd_error_dist} and application of the independence argument of Proposition~\eqref{thm:nngpind} to the multi-headed setting. 

Corollary~\ref{crl:bayes_rnd_sampling} shows that, given an estimator of the posterior predictive mean, a modified Bayesian RND setup can be used to perform direct Bayesian posterior sampling in the NTK limit. By extension, this offers a pathway to performing exact Bayesian inference through the lens of network distillation, provided that the target and predictor networks initializations are handled deliberately.

This completes our theoretical development, first showing an equivalence of RND in the NTK regime to ensemble variance and now, through specific modifications to its target function, to the generation of independent samples from exact Bayesian posterior predictive distributions.

\section{Numerical Analysis}

We proceed with a numerical analysis to validate the thus far presented results. In the following, we study how predictive RND errors relate to predictive variances of deep ensembles in practice, both in the standard and Bayesian settings. To this end, we train two-layer connected neural networks with SiLU activations \citep{elfwing2018sigmoid} on a synthetic dataset with $N=10$ train and $\tilde{N}=5000$ test samples from an isotropic Gaussian $x_i \sim \mathcal{N}(0,I_3)$. Ensemble models are fit to a toy target function, and multiheaded RND models optimized as described above. The variance of the true underlying GP is approximated with Monte-Carlo estimates of 512 independent models and a single Bayesian RND model with 512 heads, such that a small residual amount of discrepancy is to be expected. Fig.~\ref{fig:bayes_rnd} shows a stark decrease in average squared discrepancy between test evaluations of predictive ensemble variances and RND errors as model width increases, a trend in line with our theoretical derivations and present even at practical network widths. Further evaluations and details of this experiment are reported in Appendix~\ref{app:experiments}.

\begin{figure}
    \begin{center}
        \includegraphics[width=\columnwidth]{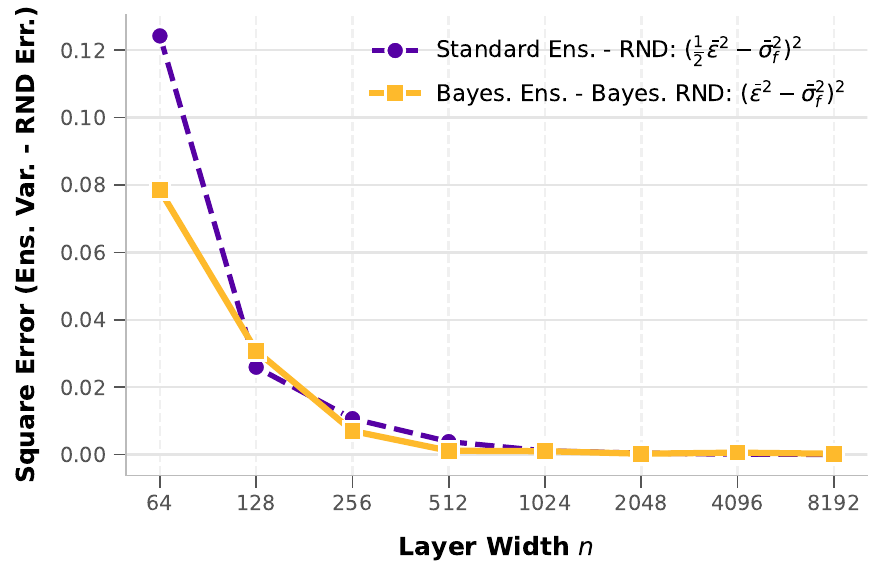}
        \caption{Test-set errors between predictive variances of (Bayesian) ensembles and self-predictive errors of (Bayesian) RND vanish with large layer widths.}
        \label{fig:bayes_rnd}
    \end{center}
\end{figure}

\section{Related Work}

A substantial body of research studies the analytical learning dynamics of deep learning, particularly in the infinite-width limit. Central to our analysis are seminal works characterizing the NNGP \citep{lee2017deep} at initialization, the dynamics-governing NTK \citep{jacotNeuralTangentKernel2020}, and the evolution of wide networks as linear models \citep{leeWideNeuralNetworks2020,  arora2019exact, chizat2018global}. This provides a theoretical framework for analytical descriptions of deep ensembles \citep{lakshminarayananSimpleScalablePredictive2017, dietterich2000ensemble}, with subsequent studies using NTK theory to precisely characterize ensemble variances under various conditions, including observation noise \citep{yang2019scaling, kobayashi2022disentangling, calvo2024epistemic}. A central line of work for our paper is the connection between deep ensembles and Bayesian inference in infinite-width NTK regime. Notably, \citet{heBayesianDeepEnsembles2020a} demonstrate how to construct ``Bayesian ensembles'', an approach we adopt to construct ``Bayesian RND'' algorithms. The broader link between deep ensembles and approximations of Bayesian posteriors has been studied extensively \citep{khan2019approximate, osawa2019practical, dangelo2021repulsive, osbandDeepExplorationRandomized2019, izmailov2021bayesian}. More recently, NTK-based approaches have been used for single-model uncertainty estimation \citep{zanger2026contextual} or ad-hoc uncertainty quantification \citep{wilson2025uncertainty}. Our work provides a theoretical basis for RND \citep{burda2018exploration}, which belongs to a class of computationally cheaper, single-model methods \citep{pathakCuriosityDrivenExplorationSelfSupervised2017, lahlou2021deup, guo2022byol, sensoyEvidentialDeepLearning2018, van2020uncertainty, rudner2022tractable, laurent2022packed, tagasovska2019single}. Moreover, Uncertainty quantification from the lens of learning dynamics is moreover widespread in reinforcement learning (RL)\citep{xiao2021understanding, cai2019neural, wai2020provably, lyle2022learning, yang2020provably}, the original application domain of RND. Notably, \citet{zanger2026universal} derive an RND-like estimator for value function uncertainty using NTK theory. More broadly, deep ensembles and Bayesian methods are widely used in RL, driving exploration \citep{osbandDeepExplorationBootstrapped2016b, chen2017ucb, osbandDeepExplorationRandomized2019, nikolovInformationDirectedExplorationDeep2019, ishfaq2021randomized, zanger2024diverse}.

% While uncertainty quantification has a rich body of literature within reinforcement learning, the application of NTK theory to RL settings is still developing. Several works have leveraged linearized learning dynamics in RL, including in overparameterized settings \citep{xiao2021understanding}, for neural networks with single or multiple layers \citep{cai2019neural, wai2020provably}, to analyze generalization \citep{lyle2022learning}, and to derive provably optimistic value functions \citep{yang2020provably}. Concurrent work to ours studies the infinite-width limit of an RND-like estimator for value function uncertainty \citep{zanger2025universal}. More broadly, deep ensembles and Bayesian methods are widely used in RL, driving exploration \citep{osbandDeepExplorationBootstrapped2016b, chen2017ucb, osbandDeepExplorationRandomized2019, nikolovInformationDirectedExplorationDeep2019, ishfaq2021randomized, zanger2024diverse}, enabling robust offline and off-policy learning \citep{an2021uncertainty, chenRandomizedEnsembledDouble2021, lee2021sunrise}, and ensuring safety \citep{lutjens2019safe, lee2022offline, hoelEnsembleQuantileNetworks2021}. Our work provides a theoretical basis for RND \citep{burda2018exploration}, which belongs to a class of computationally cheaper, single-model methods whose theoretical underpinnings are typically less understood \citep{pathakCuriosityDrivenExplorationSelfSupervised2017, lahlou2021deup, guo2022byol, sensoyEvidentialDeepLearning2018, van2020uncertainty}.

\section{Conclusions}

In this work, we have established a novel theoretical understanding of random network distillation (RND) by connecting it to the principled uncertainty frameworks of deep ensembles and Bayesian inference. By analyzing these techniques within the unifying setting of infinitely wide neural networks, we provide a clear analytical interpretation for the empirically successful RND algorithm. Our analysis yields a twofold equivalence: first, we prove that the squared error of standard RND exactly recovers the predictive variance of deep ensembles in the NTK regime. Second, we demonstrate that the RND framework is more versatile; by deliberately designing the RND target function, the resulting error signal can be made to directly mirror the centered posterior predictive distribution of an NTK-governed GP, that is, the exact posterior predictive distribution of neural networks in the infinite width limit. This ``Bayesian RND'' variant furthermore allows for posterior sampling procedures that produce i.i.d. samples from this posterior. Our work thereby unifies RND, ensembles, and Bayesian inference under the same theoretical lens from an infinite width perspective.

Crucially, our findings hold under the assumptions infinite-width and the NTK regime, a setting where networks effectively linearize and operate as kernel machines with a fixed kernel. This ``lazy'' training regime, while analytically tractable and predictive for very wide networks, does not capture the phenomenon of feature learning. The degree to which our established equivalences translate to practical, finite-width networks that learn features remains a significant open question. Conversely, this also suggest avenues for future research: deviations between RND, ensembles, and Bayesian posteriors in practice must arise from departures from the NTK regime. Characterizing specifically these deviations could lead to novel techniques and a deeper understanding of computationally efficient approaches that approximate Bayesian inference, operating well outside the kernelized infinite-width setting. Another exciting direction is the concept of \emph{target engineering} as cheap way of studying priors for Bayesian deep learning, an actively studied field that garners widespread interested from the uncertainty quantification and Bayesian deep learning community.

\clearpage
% References
\bibliography{uai2026-template}

\newpage

\onecolumn

\title{On the Equivalence of Random Network Distillation, Deep Ensembles, and Bayesian Inference \\(Supplementary Material)}
\maketitle

\appendix

\section{Limitations and Assumptions}

We provide an overview of the primary assumptions underpinning our analysis and discuss their relation to practical settings. The foremost assumption is that our analysis operates within the NTK regime. This framework presupposes the asymptotic limit of infinitely wide neural networks and a so-called NTK-parametrization of forward computations that ensures network dynamics linearize around their initialization, leading to ``lazy'' learning with kernel regression behavior. This idealized setting naturally deviates from practical implementations involving finite-width networks. Nonetheless, a significant body of work has demonstrated that predictions from NTK theory can remain remarkably accurate for sufficiently wide, modern architectures, providing a reasonable approximation of their behavior \citep[e.g.,][]{lee2020finite, seleznova2022analyzing, samarin2020empirical}.

Furthermore, our derivations assume training via full-batch gradient flow, which corresponds to gradient descent with an infinitesimal step size. This abstains from the use of stochastic minibatch optimizers, which are standard in practice. While beyond our current scope, extensions of NTK analysis to incorporate the effects of stochastic gradient noise do exist \citep[e.g.,][]{yang2019scaling, cao2019generalization, nitanda2021optimal}. Finally, our analysis considers a fixed training dataset $\datax$. This contrasts with prominent applications of RND, particularly in online reinforcement learning, where the agent interacts with an environment and learns from an inherently non-stationary data stream. Characterizing how these equivalences with ensembles and Bayesian posteriors evolve under such distribution shifts remains an important open question. 

\section{Proofs}
This section provides extended proofs for our analysis of RND.

%%%%%%%%%%%%%%%%%%%%%%%%%%%%%%%%%%%%%%%%%%%%%%%%%%%%%
% New section
%%%%%%%%%%%%%%%%%%%%%%%%%%%%%%%%%%%%%%%%%%%%%%%%%%%%%
\subsection{Ensemble Equivalence}
Our first result states the equivalence of self-predictive errors of RND and predictive variance of deep ensembles in the infinite-width NTK regime. For completeness, we also include proofs or simplified proof sketches for known results that support our analysis.

\begin{restatable}{thm}{rndwidennconverged}\citep{jacotNeuralTangentKernel2020}(Post-convergence neural network function) \label{thm:converged_f}
    In the limit of infinite layer widths $n \xrightarrow{}\infty$ and infinite time $t \xrightarrow{} \infty$, the output function of a neural network $\f(x;\param_\infty)$ with NTK parametrization according to Eq.~\ref{eq:forwardpass} is given by 
    \begin{align*}
            \f(x;\param_\infty) &= 
            \f(x;\param_0) - \ntkk_{x \datax} \ntkk_{\datax \datax}^{-1} \bigl( \datay - \f(\datax; \param_0) \bigr) \,,
    \end{align*}
    where we used the shorthand $\ntkk_{x x'} \equiv \ntkk(x,x')$.
\end{restatable}
\textbf{\textit{Proof sketch.}} By taking the infinite width limit $n \to \infty$, we obtain a linear ODE from Eq.~\eqref{eq:function_de}. Through an exponential ansatz, its explicit solution with initial condition $\f(x; \param_0)$ is given by $\f(x; \param_t) = \f(x; \param_0) + \ntkk_{x \datax}\ntkk_{\datax \datax}^{-1}(I - e^{-t \ntkk_{\datax \datax}})( \datay - \f(\datax; \param_0)).$  Assuming the training Gram matrix $\ntkk_{\datax \datax}$ is positive definite (and thus invertible), the exponential term decays to zero as $t \to \infty$, yielding the kernel regression formula in Proposition~\eqref{thm:converged_f}. See \citet{jacotNeuralTangentKernel2020} and Appendix~\ref{proofconverged_f}.

\subsubsection{Proof of Theorem ~\ref{thm:converged_f}} \label{proofconverged_f}
\begin{proof}
The proof is centered around the learning dynamics of a neural network under gradient descent, whereby we assume the limit of infinitesimal step size for simplicity. This setting is also referred to as ``gradient flow''. The driving force behind the learning dynamics of parameters $\param_t$ is gradient flow optimization on the loss 
\begin{align}
    \mathcal{L}(\param_t) &= \frac{1}{2}\|\, \f(\datax,\param_t) - \datay\,\|^2_2,
\end{align}
with the subsequent evolution of parameters by 
\begin{align}\label{eq:parameter_loss_dynamics}
    \dv{t} \param_t &= -\alpha \nabla_\param \mathcal{L}(\param_t) \,,
\end{align}
where $\alpha$ is a learning rate. From this, we can obtain the parameter space differential equation 
\begin{align} \label{eq:parameter_evolution}
    \dv{t} \param_t &= - \alpha \nabla_\param \f(\datax,\param_t)\bigl( \f(\datax,\param_t) - \datay\ \bigr)\,. 
\end{align}
In order to translate this expression to a function-space view through a first-order Taylor expansion of $\f$ around its initialization parameters $\param_0$:
\begin{align} \label{eq:flin_def}
\flin(x,\param_t) = \f(x,\param_0) + \nabla_\param \f(x,\param_0)^\top (\param_t - \param_0)\,.
\end{align}
The use of a linearized neural network function simplifies the analysis in two aspects: 1.) the linearization offers a simple translation of the parameter space evolution $\dv{t} \param_t$ to a function-space evolution and 2.) the linearized neural network function $\flin(x,\param_t)$ results in linear dynamics, simplifying the earlier derived differential equation to a linear ODE. The evolution of $\flin$ is then obtained by taking the time-derivative of Eq.~\eqref{eq:flin_def} and plugging in the parameter evolution for a linearized function from Eq.~\eqref{eq:parameter_evolution} such that 
\begin{align} \label{eq:linearized_function_de}
    \dv{t} \flin (x,\param_t) &= - \alpha \nabla_\param \f(x,\param_0)^\top \nabla_\param \f(\datax,\param_0)\bigl( \flin(\datax,\param_t) - \datay \bigr)\,.
\end{align}
Let us denote the training error of $\flin$ at time $t$ with $\delta_t = \flin (\datax,\param_t) - \datay$ and accordingly write 
\begin{align} \label{eq:training_error_de}
    \dv{t} \delta_t &= - \alpha \ntkk^0_{\datax \datax} \delta_t\,,
\end{align}
where $\ntkk^0_{\datax \datax}$ denotes the empirical tangent kernel $\ntkk^0_{\datax \datax} = \nabla_\param \f(\datax,\param_0)^\top \nabla_\param \f(\datax,\param_0)$ at initialization. The differential equation~\eqref{eq:training_error_de} is a linear ODE system to which an exponential ansatz provides the explicit solution 
\begin{align}
    \delta_t = e^{-\alpha t \ntkk^{0}_{\datax \datax}} \delta_0 \,,
\end{align} 
where $e^{\ntkk_{\datax \datax}} = \sum_{k=0}^\infty \frac{1}{k!} (\ntkk_{\datax\datax})^\dimout$ is the matrix exponential. We plug this result back in the linearized function space differential equation~\ref{eq:linearized_function_de} to obtain 
\begin{align}
    \dv{t} \flin (x,\param_t) &= - \alpha \ntkk_{x \datax}^0 e^{-\alpha t \ntkk^{0}_{\datax \datax}} \bigl(  \f(\datax,\param_0) - \datay \bigr)\,.
\end{align}
In this form, we can solve for $\flin (x,\param_t)$ directly by integration 
\begin{align} \label{eq:flin_timet}
    \flin(x,\param_t) &= \f(x,\param_0) + \int_0^t \dv{t'} \flin (x,\param_{t'}) \dd t' \\
    &= \f(x,\param_0) + \ntkk^{0}_{x \datax} (\ntkk^{0}_{\datax \datax})^{-1} \Bigl(e^{-\alpha t \ntkk^{0}_{\datax \datax}}-I \Bigr) \bigl( \f(\datax,\param_0) - \datay \bigr)\,.
\end{align}
Remarkably, the linearized and true learning dynamics become increasingly aligned with increasing neural network width. \citet{jacotNeuralTangentKernel2020} and \citet{leeWideNeuralNetworks2020} show that as network width increases, the required individual movement of parameters $\param_t - \param_0$ to effect sufficient movement in the output function $\f(x,\param_t)$ decreases. In the limit of infinite width $n \to \infty$, the linearization of $\f$ then becomes exact $\lim_{n \to \infty} \flin(x,\param_t) = \f(x,\param_t)$. Under the outlined training dynamics, the same limit furthermore causes the NTK to become deterministic (despite random weight initializations) and stationary $\lim_{n \to \infty} \ntkk_{xx'}^0 = \ntkk_{xx'}^t = \ntkk_{xx'}$. Thus, the convergenced function at time $t \to \infty$ is described by 
\begin{align} 
    \f(x,\param_\infty) &= \f(x,\param_0) - \ntkk_{x \datax} \ntkk_{\datax \datax}^{-1} \bigl( \f(\datax,\param_0) - \datay \bigr)\,.
\end{align}
\end{proof}

\subsubsection{Proof of Theorem~\ref{thm:converged_f_dist}} \label{proofconverged_f_dist}
We restate Theorem \ref{thm:converged_f_dist} for convenience. 
\rndwidenndistribution*

\textbf{\textit{Proof sketch.}} We use the fact that $\f(x;\param_\infty)$ can be written as a linear combination of the test initialization $\f(x;\param_0)$ and the training initialization $\f(\datax;\param_0)$. Both these identities are described probabilistically by the NNGP $\f(x;\param_0) \sim \mathcal{GP}(0,\kappa_{xx'})$, and $\f(\datax;\param_0) \sim \mathcal{GP}(0,\kappa_{\datax \datax})$. Applying a linear transformation to a GP yields another GP \citep{rasmussen2006gp}, meaning $\f(x;\param_\infty)$ also follows a GP. Propagating the prior covariance $\nngpk$ through the linear transformation described by Proposition~\ref{thm:converged_f} reveals the expression for the post-convergence covariance function $\Sigma^\f_{\datax_T \datax_T}(\param_\infty)$ given in Theorem~\ref{thm:converged_f_dist}. 

\begin{proof}
The proof builds on the previous result of Proposition~\ref{thm:converged_f} providing a closed-form expression for the post-convergence function as a deterministic function of its initialization, here evaluated for a set of test points $\datax_T$
\begin{align} 
    \f(\datax_T,\param_\infty) &= \f(\datax_T,\param_0) - \ntkk_{\datax_T \datax} \ntkk_{\datax \datax}^{-1} \bigl( \f(\datax,\param_0) - \datay \bigr)\,.
\end{align}
To be precise, the post-convergence predictions $\f(\datax_T,\param_\infty)$ can be written as an affine transformation of the vector $(\f(\datax_T,\param_0), \f(\datax,\param_0)^\top)^\top$. This yields the block matrix equation 
\begin{align} \label{eq:blockeq}
    & \nonumber
    \begin{pmatrix}
        \f(\datax_T,\param_\infty) \\
        \f(\datax,\param_\infty)
    \end{pmatrix} 
     = \\ 
    & \begin{pmatrix}
        I & -\ntkk(\datax_T, \datax)\ntkk(\datax,\datax)^{-1} \\
        0 & 0
    \end{pmatrix} 
    \begin{pmatrix}
        \f(\datax_T,\param_0) \\
        \f(\datax,\param_0)
    \end{pmatrix}
    +
    \begin{pmatrix}
        \ntkk(\datax_T, \datax)\ntkk(\datax, \datax)^{-1}\datay \\
        \datay
    \end{pmatrix} \,.
\end{align}
We recall that, at initialization, neural networks in the infinite width limit distribute to a GP called NNGP \citep{lee2017deep} as 
\begin{align}
    \f(\datax_T,\param_0) \sim \mathcal{GP}(0, \nngpk_{\datax_T \datax_T})\, \quad \text{where} \quad \nngpk_{\datax_T \datax_T} = \mathbb{E}_{\param_0}[\f(\datax_T,\param_0) \f(\datax_T,\param_0)^\top]\,.
\end{align}
The block eq.~\eqref{eq:blockeq} thus describes an affine transformation of a GP itself. We have that affine transformations of multivariate Gaussian random variables $X \sim \mathcal{N}(\mu_X, \Sigma_X)$ with $Y = a + BX$ distribute Gaussian themselves with $Y \sim \mathcal{N}(a+B\mu_X, \,\, B\Sigma_XB^\top)$. Application to Eq.~\ref{eq:blockeq} and rearrangement then yields the post-convergence GP with mean and covariance
\begin{align}
    & \mathbb{E}[\f(\datax_T, \param_\infty)] = \ntkk_{\datax_T\datax}\ntkk_{\datax \datax}^{-1}\datay \,, \\
    & \nonumber \Sigma^\f_{\datax_T \datax_T}(\param_\infty) = \\
    & \quad \nngpk_{\datax_T \datax_T} + \ntkk_{\datax_T \datax} \ntkk_{\datax \datax}^{-1} \nngpk_{\datax \datax} \ntkk_{\datax \datax}^{-1}\ntkk_{\datax \datax_T} - \bigl( \ntkk_{\datax_T \datax}\ntkk_{\datax \datax}^{-1}\nngpk_{\datax \datax_T} + \text{h.c.} \bigr) \,,
\end{align}
where h.c. refers to the Hermitian conjugate of the preceding term. This completes the proof.
\end{proof}

\subsubsection{Proof of Theorem~\ref{thm:rnd_error_dist}} \label{proofrnd_error_dist}
We restate Theorem \ref{thm:rnd_error_dist} for convenience. 
\rnderrordist*
\begin{proof}
This proposition considers the post-convergence distribution of self-predictive errors as produced by RND. The online predictor $\fii(x; \paramii_t)$ undergoes learning dynamics under the same conditions as outlined in the derivation of Proposition~\ref{thm:converged_f}, albeit with the self-predictive loss 
\begin{align}
    \mathcal{L}(\paramii_t) &= \frac{1}{2}\|\, \fii(\datax,\paramii_t) - \fiii(\datax,\paramiii_0)\,\|^2_2 \,.
\end{align}
This, by analogy to Theorem~\ref{thm:converged_f}, implies that the online predictor $\fii(x;\paramii_t)$ converges as $t \to \infty$ to the function 
\begin{align}
    \fii(x,\paramii_\infty) &= \fii(x,\paramii_0) - \ntkk_{x \datax} \ntkk_{\datax \datax}^{-1} \bigl( \fii(\datax,\paramii_0) - \fiii(\datax, \paramiii_0) \bigr)\,.
\end{align}
For a set of test points $\datax_T$, the error $\err(\datax_T; \paramii_\infty, \paramiii_0) = \fii(\datax_T; \paramii_\infty) - \fiii(\datax_T; \paramiii_0)$ at convergence can thus be written as the affine transformation
\begin{align}
    \err(\datax_T; \paramii_\infty, \paramiii_0) &= \err(\datax_T; \paramii_0, \paramiii_0)  - \ntkk_{\datax_T \datax} \ntkk_{\datax \datax}^{-1} \err(\datax; \paramii_0, \paramiii_0)\,.
\end{align}
and the corresponding block matrix equation
\begin{align} \label{eq:blockeq_rnd}
    \begin{pmatrix}
        \err(\datax_T; \paramii_\infty, \paramiii_0) \\
        \err(\datax; \paramii_\infty, \paramiii_0)
    \end{pmatrix} 
    &=
    \begin{pmatrix}
        I & -\ntkk_{\datax_T \datax} \ntkk_{\datax \datax}^{-1} \\
        0 & 0
    \end{pmatrix} 
    \begin{pmatrix}
        \err(\datax_T; \paramii_0, \paramiii_0) \\
        \err(\datax; \paramii_0, \paramiii_0)
    \end{pmatrix} \,.
\end{align}
The errors accordingly are themselves Gaussian with $\err(\datax_T; \paramii_\infty, \paramiii_0) \sim \mathcal{GP}(0, \nngpk^\err_{\datax_T \datax_T})$ where $\nngpk^\err_{\datax_T \datax_T} = \mathbb{E}_{\paramii_0, \paramiii_0}[\err(\datax_T;\paramii_0, \paramiii_0) \err(\datax_T;\paramii_0, \paramiii_0)^\top]$. The latter term describes the distribution of self-predictive errors at initialization, which is a simple sum of two independent NNGP $\err(\datax_T;\paramii_0, \paramiii_0) = \fii(\datax_T; \paramii_0) - \fiii(\datax_T; \paramiii_0)$ such that $\nngpk^\err_{\datax_T \datax_T} = \nngpk^\fii_{\datax_T \datax_T} + \nngpk^\fiii_{\datax_T \datax_T}$, completing the proof.\
\end{proof}

\subsubsection{Proof of Proposition~\ref{thm:nngpind}} \label{proofnngpind}
Before treating Proposition~\ref{thm:nngpind} we first derive two known results concerning the independence and recursive character of the NNGP kernel and the NTK. We assume forward computations of $\f(x;\param_t)$ are defined according to Eq.~\ref{eq:forwardpass}. To avoid confusion with indices $i,j$ we will in this section use the notation $\nngpk(x,x')$ rather than $\nngpk_{xx'}$ to denote the function inputs $x,x'$ (and similarly for $\ntkk(x,x')$). 

\begin{restatable}{prp}{rndindependencenngp}\citep{lee2017deep} \label{thm:nngpindproof}(Recursive NNGP formulation) 
At initialization $t=0$ and in the limit $n \to \infty$, the $i$-th output at layer $l$, $z_i^l(x; \param_0^{\leq l})$, converges to a GP with zero mean and covariance function $\nngpk_{ii}^l(x,x')$ given by
\begin{align}
    \nngpk_{ii}^1(x,x') &= \frac{\sigma_w^2}{n_0}x^\top x' + \sigma_b^2,\quad \text{and}\quad k_{ij}^1(x,x')=0,\quad \text{if } i\neq j \,, \\
    \nngpk_{ii}^{l}(x,x') &= \sigma_b^2 + \sigma_w^2 \mathbb{E}_{z_i^{l-1}\sim \mathcal{GP}(0,\nngpk_{ii}^{l-1})}[\nonlinearity(z_i^{l-1}(x; \param_0^{\leq l-1}))\nonlinearity(z_i^{l-1}(x'; \param_0^{\leq l-1}))] \,, \\
    & \qquad \qquad \qquad \qquad \,\,\, \text{and} \quad \nngpk_{ij}^{l}(x,x') = 0, \quad \text{if } i\neq j\,,
\end{align}
and we have $\nngpk_{ii}^l(x,x')=\nngpk^l(x,x')\, ,\quad \forall i$.
\end{restatable}

\begin{proof}
We prove the proposition by induction. The induction assumption is that if outputs at layer $l-1$ satisfy a GP structure
\begin{align}
    z_i^{l-1} \sim \mathcal{GP}(0,\nngpk^{l-1}),
\end{align}
with the covariance function defined as
\begin{align}
    \nngpk_{ij}^{l-1}(x,x') &= \mathbb{E}[z_i^{l-1}(x; \param_0^{\leq l-1}) z_j^{l-1}(x'; \param_0^{\leq l-1})] = \begin{cases}
        k^{l-1}(x,x') &\quad \text{if } i=j\,, \\
        0 & \quad \text{if } i \neq j\,, \\
    \end{cases}
\end{align}
then, outputs at layer $l$ follow
\begin{align}
    z_i^{l}(x)\sim \mathcal{GP}(0,\nngpk^{l}),
\end{align}
where the NNGP kernel at layer $l$ is given by:
\begin{align}
    \nngpk_{ii}^{l}(x,x') &= \mathbb{E}[z_i^{l}(x; \param_0^{\leq l}) z_i^{l}(x'; \param_0^{\leq l})] = \nngpk^{l}(x,x'), \quad \forall i, \\
    \nngpk_{ij}^{l}(x,x') &= \mathbb{E}[z_i^{l}(x; \param_0^{\leq l}) z_j^{l}(x'; \param_0^{\leq l})] = 0,\quad \text{if } i\neq j.
\end{align}
with the recursive definition 
\begin{align}
   \nngpk^{l}(x,x') = \sigma_b^2 + \sigma_w^2 \mathbb{E}_{z_i^{l-1}\sim \mathcal{GP}(0,k^{l-1})}[\nonlinearity(z_i^{l-1}(x; \param_0^{\leq l-1}))\nonlinearity(z_i^{l-1}(x'; \param_0^{\leq l-1}))].
\end{align}
\emph{Base case $(l=1)$}. At layer $l=1$ we have:
\begin{align}
    z_i^1(x; \param_0^{\leq 1})=\frac{\sigma_w}{\sqrt{n_0}}\sum_{j=1}^{n_0} w_{ij}^1 x_j + \sigma_b b_i^1 \,.
\end{align}
This is an affine transform of Gaussian random variables; thus, $z_i^1(x; \param_0^{\leq 1})$ distributes Gaussian with
\begin{align}
    z_i^1(x)\sim \mathcal{GP}(0,\nngpk^1),    
\end{align}
with kernel
\begin{align}
    \nngpk^1(x,x')=\frac{\sigma_w^2}{n_0}x^\top x' + \sigma_b^2 = \nngpk^1_{ii}(x,x')\,, \quad\text{and}\quad \nngpk_{ij}^1=0,\quad \text{if }\, i\neq j\,,
\end{align}
where the independence follows from the fact that $z_i^1(x; \param_0^{\leq 1})$ is computed from separate, independent rows of weights and biases. 

\emph{Induction step $l>1$.} For layers $l>1$ we have
\begin{align}
z_i^{l}(x; \param_0^{\leq l})=\sigma_b b_i^{l}+\frac{\sigma_w}{\sqrt{n_{l-1}}}\sum_{j=1}^{n_{l-1}}w_{ij}^{l} x_j^{l}(x),
\quad x_j^{l}(x)=\phi(z_j^{l-1}(x; \param_0^{\leq l-1})) \,.
\end{align}
By the induction assumption, $z_j^{l-1}(x; \param_0^{\leq l-1})$ are generated by independent GP. Hence, $x_i^{l}(x)$ and $x_{j}^{l}(x)$ are independent for $i\neq j$. Consequently, $z_i^{l}(x; \param_0^{\leq l})$ is a sum of independent random variables. By the CLT (as $n_1, \dots, n_{L} \rightarrow\infty$) the tuple $\{z_i^{l}(x; \param_0^{\leq l}),z_{i}^{l}(x'; \param_0^{\leq l})\}$ tends to be jointly Gaussian, with covariance given by:
\begin{align}
    & \nonumber \mathbb{E}[z_i^{l}(x; \param_0^{\leq l}) z_i^{l}(x'; \param_0^{\leq l})] 
    = \\
    & \quad \sigma_b^2 + \sigma_w^2\mathbb{E}_{z_i^{l-1}\sim \mathcal{GP}(0,\nngpk^{l-1})}[\nonlinearity(z_i^{l-1}(x; \param_0^{\leq l-1}))\nonlinearity(z_i^{l-1}(x'; \param_0^{\leq l-1}))] \,.
\end{align}
Moreover, as $z_i^{l}$ and $z_j^{l}$ for $i\neq j$ are defined through independent rows of the parameters $w^{l}, b^{l}$ and independent pre-activations $x^{l}(x)$, we have
\begin{align}
    \nngpk_{ij}^{l} = \mathbb{E}[z_i^{l}(x) z_j^{l}(x')]=0,\quad \text{if }\, i\neq j,  
\end{align}
and thus completing the proof.
\end{proof}

\begin{restatable}{prp}{rndindependencentk}\citep{jacotNeuralTangentKernel2020} \label{thm:ntkind}(Recursive NTK formulation)
In the limit $n \to \infty$, the neural tangent kernel $\ntkk^l_{ii}(x,x')$ of the $i$-th output $z_i^l(x; \param_0^{\leq l})$ at layer $l$, defined as the gradient inner product
\begin{align}
    \ntkk^l_{ii}(x,x') = \nabla_{\param^l} z_i^l(x; \param_0^{\leq l})^\top \nabla_{\param^l} z_i^l(x'; \param_0^{\leq l}) \,,
\end{align}
is given recursively by 
\begin{align}
    \ntkk_{ii}^1(x,x') &= \nngpk_{ii}^1 (x,x') = \frac{\sigma_w^2}{n_0}x^\top x' + \sigma_b^2,\quad\text{and}\quad \ntkk_{ij}^1(x,x')=0,\quad \text{if }\, i\neq j\,, \\
    \ntkk_{ii}^{l}(x,x') &= \ntkk_{ii}^{l-1} (x,x') \dot{\nngpk}_{ii}^{l-1} (x,x') + \nngpk_{ii}^{l}(x,x'), \\
\end{align}
where
\begin{align}
    \dot{\nngpk}_{ii}^{l} (x,x') &= \sigma_w^2 \mathbb{E}_{z_i^{l-1}\sim \mathcal{GP}(0,\nngpk_{ii}^{l-1})}[\dot{\nonlinearity}(z_i^{l-1}(x; \param_0^{\leq l-1}))\dot{\nonlinearity}(z_i^{l-1}(x'; \param_0^{\leq l-1}))]\,,
\end{align}
and
\begin{align}
    \ntkk_{ij}^{l}(x,x') &= \nabla_{\param^l} z_i^{l}(x; \param_0^{\leq l})^\top \nabla_{\param^l} z_j^l(x'; \param_0^{\leq l}) = 0 \quad \text{if }\, i\neq j.
\end{align}
\end{restatable}

\begin{proof}
The proof is by induction. The induction assumption is that if gradients satisfy at layer $l-1$
\begin{align}
& \nonumber \ntkk_{ij}^{l-1}(x,x') = \\ 
& \quad \nabla_{\param^{l-1}} z_i^{l-1}(x; \param_0^{\leq l-1})^\top \nabla_{\param^{l-1}} z_j^{l-1}(x'; \param_0^{\leq l-1}) = 
    \begin{cases}
        \ntkk^{l-1}(x,x') & \quad \text{if } i=j, \\
    0 &\quad \text{if } i\neq j,
    \end{cases}
\end{align}
then at layer $l$ we have
\begin{align}
    \ntkk_{ij}^{l}(x,x')&= \begin{cases}
        \ntkk_{ii}^{l-1}(x,x')\dot{\nngpk}_{ii}^{l}(x,x')+\nngpk_{ii}^{l}(x,x') & \quad \text{if } i=j \,, \\ 
        0 & \quad \text{if } i \neq j \,. \\
    \end{cases}
\end{align}

\emph{Base case ($l=1$).} At layer $l=1$, we have
\begin{align}
z_i^1(x; \param_0^{\leq 1})&=\sigma_b b_i^1+\frac{\sigma_w}{\sqrt{n_0}}\sum_j^{n_0} w_{ij}^1 x_j,
\end{align}
and the gradient inner product is given by:
\begin{align}
\nabla_{\param^1}z_i^1(x; \param_0^{\leq 1})^\top\nabla_{\param^1}z_i^1(x'; \param_0^{\leq 1})=\frac{\sigma_w^2}{n_0}x^\top x'+\sigma_b^2=\nngpk_{ii}^1(x,x').
\end{align}
\emph{Inductive step ($l>1$).} For layers $l>1$, we split parameters $\param^{l}=\param^{l-1}\cup\{w^{l},b^{l}\}$ and split the inner product by
\begin{align}
\ntkk_{ii}^{l}(x,x')&= \underbrace{ \nabla_{\param^{l-1}}z_i^{l}(x; \param_0^{\leq l})^\top\nabla_{\param^{l-1}}z_i^{l}(x'; \param_0^{\leq l}) }_{l.h.s}
+
\underbrace{\nabla_{ \{ w^{l},b^{l} \} }z_i^{l}(x; \param_0^{\leq l})^\top\nabla_{ \{ w^{l},b^{l} \} } z_i^{l}(x; \param_0^{\leq l})}_{r.h.s}.
\end{align}
Note that the above $r.h.s$ involves gradients w.r.t. last-layer parameters, i.e. the post-activation outputs of the previous layer, and by the same arguments as in the NNGP derivation of Proposition~\ref{thm:nngpindproof}, this is a sum of independent post activations s.t. in the limit $n_{l-1} \xrightarrow{} \infty$
\begin{align}
    \nabla_{\{ w^{l},b^l \}} z_i^{l}(x; \param_0^{\leq l})^\top\nabla_{\{ w^{l}, b^l \}} z_j^{l} (x'; \param_0^{\leq l})&= \begin{cases}
        k_{ii}^{l}(x,x'), & \quad i=j, \\
        0, & \quad i\neq j.
    \end{cases}
\end{align}
For the $l.h.s.$, we first apply chain rule to obtain
\begin{align}
\nabla_{\param^{l-1}}z_i^{l}(x; \param_0^{\leq l})=\frac{\sigma_w}{\sqrt{n_{l-1}}}\sum_j^{n_{l-1}} w_{ij}^{l}\dot{\nonlinearity}(z_j^{l-1}(x; \param_0^{\leq l-1}))\nabla_{\param^{l-1}}z_j^{l-1}(x; \param_0^{\leq l-1}) \,.
\end{align}
The gradient inner product of outputs $i$ and $j$ thus reduces to 
\begin{align} \nonumber
& \nabla_{\param^{l-1}}z_i^{l}(x; \param_0^{\leq l})^\top \nabla_{\param^{l-1}}z_j^{l}(x'; \param_0^{\leq l}) = \\ 
& \qquad \frac{\sigma_w^2}{n_{l-1}}\sum_k^{n_{l-1}} w_{ik}^{l} w_{jk}^{l}\dot{\nonlinearity} (z_k^{l-1} (x; \param_0^{\leq l-1}) ) \dot{\nonlinearity}(z_k^{l-1} (x'; \param_0^{\leq l-1}) ) \ntkk_{kk}^{l-1}(x,x')\,.
\end{align} 
By the induction assumption $\ntkk_{kk}^{l-1}(x,x')=\ntkk^{l-1}(x,x')$ and again by the independence of the rows $w^l_i$ and $w_j^l$ for $i\neq j$, the above expression converges in the limit $n_{l-1} \xrightarrow{} \infty$ to an expectation with 
\begin{align}
    \ntkk_{ij}^{l}(x,x') = \begin{cases}
        \ntkk^{l-1}(x,x') \dot{\nngpk}_{ii}^{l}(x,x') + \nngpk_{ii}^{l}(x,x') & \quad i=j, \\
        0 &\quad i\neq j \,,
    \end{cases}
\end{align}
thereby completing the proof.
\end{proof}

We now restate Proposition~\ref{thm:nngpind} for convenience. 
\rndindependenceprp*
\begin{proof}
    We begin by deriving the training dynamics for the output $\f_i(x;\param_t)$ analogously to the proof of Proposition~\ref{thm:converged_f}. We denote by $\datay_i$ the labels used to train the function $\f_i(x;\param_t)$. By Proposition~\ref{thm:ntkind}, the training dynamics of $\f_i(x;\param_t)$ and $\f_j(x;\param_t)$ are decoupled for $i\neq j$ and we can thus derive Eq.~\ref{eq:flin_timet} analogously for individual output heads $i$. Taking the infinite width limit, we obtain at time $t$
    \begin{align}
        \f_i(x;\param_t) = \f_i(x;\param_0) + \ntkk_{ii}(x, \datax) \ntkk_{ii}(\datax, \datax)^{-1} \Bigl( e^{-\alpha t \ntkk_{ii}(\datax, \datax)} - I\Bigr) (\f_i(\datax; \param_0)-\datay_i) \,.
    \end{align}
    Thus, the output head $\f_i(x;\param_t)$ at time $t$ is a deterministic function of its own initialization only, which itself is characterized by a GP $\f_i(x;\param_0) \sim \mathcal{GP}(0, \nngpk_{ii}(x,x'))$ that is independent of output heads $j \neq i$ by Proposition~\ref{thm:nngpindproof}. And thus, since $\f_i(x;\param_t)$ is an affine transform of its own independent initialization terms $\f_i(x;\param_0)$ and $\f_i(\datax;\param_0)$, it too must follow an independent GP with $\mathbb{E}_{\param_0}[\f_i(x;\param_t)\f_i(x';\param_t)] = \Sigma(x,x';\param_t)$ and in particular $\mathbb{E}_{\param_0}[\f_i(x;\param_t)\f_j(x';\param_t)] = 0$ if $i\neq j$.
\end{proof}

\subsubsection{Proof of Theorem~\ref{thm:rndequivalence}} \label{proofrndequivalence}
We restate Theorem~\ref{thm:rndequivalence} for convenience. 
\rndequivalence*

\begin{proof}
The proof follows by combining the results of Propositions~\eqref{thm:rnd_error_dist} and \eqref{thm:nngpind}. We define a multiheaded RND predictor with $K$ output heads $\{ \fii_i(x, \paramii_t) \}_{i=1}^\dimout$ and a fixed multiheaded target network $\{ \fiii_i(x_t;\paramiii_0) \}_{i=1}^\dimout$ of equivalent architecture as $\fii_i$ (i.e., both corresponding to the same NTK $\ntkk$) with the corresponding prediction errors $\{ \err_i(x;\paramii_t, \paramiii_0) \}_{i=1}^\dimout$ accordingly. Let $\fii_i(x, \paramii_t)$ be trained such that each head $i$ is trained to match the $i$-th target output $\fiii_i(x;\paramiii_0)$. 
 
By Proposition~\ref{thm:nngpind}, the predictions of online predictor heads $\{ \fii_i(x, \paramii_t) \}_{i=1}^\dimout$ at time $t$ and fixed target networks $\{ \fiii_i(x_t;\paramiii_0) \}_{i=1}^\dimout$ are each mutually independent with 
\begin{align}
    \mathbb{E}_{\paramii_0}[\fii_i(x;\paramii_t)\fii_j(x;\paramii_t)] &= 0\,, \quad \text{if }\, i \neq j\,, \\
    & \nonumber \text{and} \\
    \mathbb{E}_{\paramiii_0}[\fiii_i(x;\paramiii_0)\fiii_j(x;\paramiii_0)] &= 0\,, \quad \text{if }\, i \neq j\,.
\end{align}
As a consequence, we also have that 
\begin{align}
    \mathbb{E}_{\paramii_0, \paramiii_0} [\err_i(x;\paramii_t, \paramiii_0) \err_j(x;\paramii_t, \paramiii_0)] = 0\,, \quad \text{if }\, i \neq j\,.
\end{align}
As previously established in the proof of Proposition~\ref{thm:nngpind}, the multi-headed functions $\{ \err_i(x;\paramii_t, \paramiii_0) \}_{i=1}^\dimout$ follow equivalent learning dynamics as their scalar-output counterparts. The post-convergence distribution of individual heads $\err_i(x;\paramii_\infty, \paramiii_0)$ must therefore equal the scalar-output post-convergence distribution established in Theorem~\ref{thm:rnd_error_dist}. Consequently, the errors $\{ \err_i(x;\paramii_t, \paramiii_0) \}_{i=1}^\dimout$ are independent and identically distributed draws from a Gaussian with mean and covariance 
\begin{align*}
    \mathbb{E}[\err(x, \paramii_\infty,\paramiii_0)] &= 0 \,, \\
    \Sigma^\err_{x x'}(\paramii_\infty,\paramiii_0) &= \nngpk^{\err}_{x x'} + \ntkk_{x \datax} \ntkk_{\datax \datax}^{-1} \nngpk^\err_{\datax \datax} \ntkk_{\datax \datax}^{-1}\ntkk_{\datax x'} - \bigl( \ntkk_{x \datax}\ntkk_{\datax \datax}^{-1}\nngpk^\err_{\datax x'} + \text{h.c.} \bigr) \,,
\end{align*} 
where $\nngpk^{\err}_{x x'} = \nngpk^{\fii}_{x x'} + \nngpk^{\fiii}_{x x'}$. The sample mean square $\frac{1}{2} \bar{\err}^2(x; \paramii_\infty,\paramiii_0) = \frac{1}{2K} \sum_{i=1}^{K} \err_i^2(x;\paramii_\infty,\paramiii_0)$ is then known to follow a scaled Chi-squared distribution with $K$ degrees of freedom 
\begin{align}
    \smallfrac{1}{2} \bar{\err}^2(x; \paramii_\infty,\paramiii_0) \sim \frac{ \smallfrac{1}{2}\Sigma^\err_{x x}(\paramii_\infty,\paramiii_0) }{K} \chi^2(K)
\end{align}
where $\Sigma^\err_{x x} (\paramii_\infty,\paramiii_0)$ is the variance of the GP described in Theorem~\ref{thm:rnd_error_dist}. 

Conversely, a set of $K+1$ independent neural networks arranged to a deep ensemble $\{\f(x;\param_\infty^i)\}_{i=1}^{K+1}$ in the infinite width limit $n \to \infty$ and at convergence $t \to \infty$ are by definition i.i.d. samples from the GP described in Theorem~\ref{thm:converged_f_dist}. As before, the empirical variance defined as $\bar{\sigma}_f^2 (x; \param_\infty^{i\dots K+1}) = \frac{1}{K}\sum_{i=1}^{K+1} \bigl( \f(x;\param^i_\infty) - \frac{1}{K+1}\sum_{j=1}^{K+1} \f(x;\param^j_\infty) \bigr)^2$ distributes as a scaled Chi-squared distribution with $K$ degrees of freedom 
\begin{align}
    \bar{\sigma}_f^2 (x; \param_\infty^{i\dots K+1}) \sim \frac{ \Sigma^\f_{x x}(\param_\infty) }{K} \chi^2(K) \,,
\end{align}
where $\Sigma^\f_{x x}(\param_\infty)$ is the variance of the GP described in Theorem~\ref{thm:converged_f_dist}. 

Finally, as we assume equal architecture and i.i.d. initialization of $\fii$, $\fiii$, and $\f$, we have that $\nngpk^{\err}_{x x'} = \nngpk^{\fii}_{x x'} + \nngpk^{\fiii}_{x x'} = 2\nngpk^{\fii}_{x x'} = 2\nngpk_{xx'}$ and accordingly $\smallfrac{1}{2}\Sigma^\err_{x x}(\paramii_\infty,\paramiii_0) = \Sigma^\f_{x x}(\param_\infty)$, completing the proof. 
\end{proof}

\subsection{Posterior Equivalence} 
This section contains proofs for results pertaining to the equivalence of self-predictive errors of ``Bayesian RND'' and the variance of Bayesian posterior predictive distributions of neural networks in the infinite width limit. 
\subsubsection{Proof of Proposition ~\ref{thm:bayesrndrarget}} \label{proofbayesrndrarget}
We restate Proposition \ref{thm:bayesrndrarget} for convenience. 
\rndntkgpinit*
\begin{proof}
    The proof will show that in the limit $n \to \infty$ the function $\tilde{\fiii}(x;\paramii_0, \paramiii_0)$ converges to a GP $\tilde{\fiii}(x;\paramii_0, \paramiii_0) \sim \mathcal{GP}(0, \ntkk^{\leq L-1}_{xx'})$ by Lévy's continuity theorem, which we recall informally below. 
    \begin{restatable}{thm}{levycontinuity}\label{thm:levycontinuity}(Lévy's continuity theorem) 
        Let $\{ Z_n\}_{n=1}^\infty$ be a sequence of $\reals^n$-valued random variables. Their characteristic functions $\varphi_{Z_n}(t)$  for some $t \in \reals^n$ are given by 
        \begin{align}
            \varphi_{Z_n}(t) = \mathbb{E}[e^{i t^\top Z_n}] \,,
        \end{align}
        where $i$ is the imaginary unit. If in the limit $n \to \infty$ the sequence of characteristic functions converges pointwise to a function 
        \begin{align}
            \varphi_{Z_n}(t) \to \varphi(t) \quad \forall t \in \reals^n \,,
        \end{align}
        then $Z_n$ converges in distribution to a random variable $Z$ 
        \begin{align}
            Z_n \overset{D}{\to} Z \,,
        \end{align}
        whose characteristic function is $\varphi_Z(t) = \varphi(t)$
    \end{restatable}
    \noindent Rigorous proof can be found for example in \citet{durrett2019probability}. 

    We begin by rewriting the function $\tilde{\fiii}(x;\paramii_0, \paramiii_0)$ as a linear model with 
    \begin{align}
        \tilde{\fiii}(x;\paramii_0, \paramiii_0) &= \nabla_\paramii \fii(x;\paramii_0)^\top \paramiii_0^* \\
        &= \nabla_{\paramii^{\leq L-1}} \fii(x;\paramii_0)^\top \paramiii_0^{\leq L-1} \,.
    \end{align}
    Since $\paramiii_0^{\leq L-1}$ is an independent draw from $\paramii_0$ by assumption, $\tilde{\fiii}(x;\paramii_0, \paramiii_0)$ is a random affine transform of the Gaussian vector $\paramiii_0^{\leq L-1}$. For more precise treatment of the distribution of $\tilde{\fiii}(x;\paramii_0, \paramiii_0)$, we write $\tilde{G}(\datax_T)$ to denote the random variable corresponding to the function evaluations of $\tilde{\fiii}$ on a test set $\datax_T$. Conditioned on $\paramii_0$ (i.e., fixing the affine transform), we thus have that $\tilde{G}(\datax_T)|\paramii_0 \sim \mathcal{GP}(0, \ntkk^{\leq L-1}_{0, \datax_T \datax_T})$, where $\ntkk^{\leq L-1}_{0, \datax_T \datax_T} = \nabla_{\paramii^{\leq L-1}} \fii(\datax_T;\paramii_0)^\top \nabla_{\paramii^{\leq L-1}} \fii(\datax_T;\paramii_0)$ is the empirical NTK matrix of $\fii$. Note that this statement holds irrespective of the network width $n$. 

    Next, we show that the unconditional law of $\tilde{G}(\datax_T)$, too, tends to a GP in the limit $n \to \infty$. To this end, we examine the distribution of the unconditioned random vector $\tilde{G}(\datax_T)$ through its characteristic function
    \begin{align}
        \varphi_{\tilde{G}(\datax_T)}(t) = \mathbb{E}[e^{i t^\top \tilde{G}(\datax_T)}]\,.
    \end{align}
    This characteristic function $\varphi_{\tilde{G}(\datax_T)}(t)$ uniquely defines the distribution of ${\tilde{G}(\datax_T)}$ \citep{durrett2019probability}. By the law of total expectation, the characteristic function of the unconditional variable $\tilde{G}(\datax_T)$ can then be written as 
    \begin{align} \label{eq:total_expectation_characteristic}
        \varphi_{\tilde{G}(\datax_T)}(t) = \mathbb{E}_{\paramii_0}\bigl[ \mathbb{E}[e^{i t^\top \tilde{G}(\datax_T)}|\paramii_0] \bigr]  \,.
    \end{align}
    As stated above, the conditional distribution of $\tilde{G}(\datax_T)|\paramii_0$ is a zero-mean Gaussian with the empirical covariance $\ntkk^{\leq L-1}_{0, \datax_T \datax_T}$, to which we can show the conditional characteristic function is given by \citep{durrett2019probability} 
    \begin{align}
        \mathbb{E}[e^{it^\top\tilde{G}(\datax_T)}|\paramii_0] = e^{- \frac{1}{2} t^\top \ntkk^{\leq L-1}_{0, \datax_T \datax_T} t} \,.
    \end{align}
    Plugging this back into Eq.~\ref{eq:total_expectation_characteristic} gives
    \begin{align}
        \varphi_{\tilde{G}(\datax_T)}(t) = \mathbb{E}_{\paramii_0}[e^{- \frac{1}{2} t^\top \ntkk^{\leq L-1}_{0, \datax_T \datax_T} t}] \,.
    \end{align}
    We now use the known result by \citet{jacotNeuralTangentKernel2020} that, as $n \to \infty$ we have that $\ntkk_{0, \datax_T \datax_T} \to \ntkk_{\datax_T \datax_T}$ in probability and accordingly $\ntkk^{\leq L-1}_{0, \datax_T \datax_T} \to \ntkk^{\leq L-1}_{\datax_T \datax_T}$ converges to a deterministic kernel matrix. Moreover, since the Gram matrix $\ntkk^{\leq L-1}_{0, \datax_T \datax_T}$ is positive semidefinite in general, the term $e^{- \frac{1}{2} t^\top \ntkk^{\leq L-1}_{0, \datax_T \datax_T} t}$ is bounded and continuous. By bounded convergence \citep{durrett2019probability}, we can then conclude that we also have convergence of the characteristic function through 
    \begin{align}
        \lim_{n \to \infty} \varphi_{\tilde{G}(\datax_T)}(t) &= \lim_{n \to \infty} \mathbb{E}_{\paramii_0}[e^{- \frac{1}{2} t^\top \ntkk^{\leq L-1}_{0, \datax_T \datax_T} t}] \\ 
        &= e^{- \frac{1}{2} t^\top \ntkk^{\leq L-1}_{\datax_T \datax_T} t} \,.
    \end{align}
    As stated earlier, for a Gaussian random vector $Z$ with $Z \sim \mathcal{GP}(0, \ntkk_{\datax_T \datax_T}^{\leq L-1})$ its characteristic function is given by $e^{- \frac{1}{2} t^\top \ntkk^{\leq L-1}_{\datax_T \datax_T} t}$. Invoking Lévy's continuity theorem, the pointwise convergence of $\varphi_{\tilde{G}(\datax_T)}(t)$ to this exact limit $\varphi_{\tilde{G}(\datax_T)}(t) \to e^{- \frac{1}{2} t^\top \ntkk^{\leq L-1}_{\datax_T \datax_T} t}$ then implies convergence in distribution of $\tilde{G}(\datax_T) \overset{D}{\to} Z$ and we can thus conclude $\tilde{\fiii}(x;\paramii_0, \paramiii_0) \sim \mathcal{GP}(0, \ntkk^{\leq L-1}_{x x'})$.
\end{proof}

\subsubsection{Proof of Theorem ~\ref{thm:bayes_rnd_error_dist}} \label{proofbayes_rnd_error_dist}
We restate Proposition \ref{thm:bayesrndrarget} for convenience. 
\rndbayesrnderrordist*
\begin{proof}
    The result follows from the independence of the two GP of interest in the limit $n \to \infty$. First, this is $\tilde{\fiii}(x;\paramii_0, \paramiii_0) \sim \mathcal{GP}(0, \ntkk_{xx'}^{\leq L-1})$ and second, $\fii(x;\paramii_0) \sim \mathcal{GP}(0, \ntkk_{xx'}^{L})$. In the following, we will show that the two GPs are in the limit $n \to \infty$ independent processes such that Eq.~\ref{eq:bayeserrorprior} applies. 

    We first write for any two points $x,x'$ the covariance 
    \begin{align}
        \mathrm{Cov}[\tilde{\fiii}(x;\paramii_0, \paramiii_0), \fii(x';\paramii_0)] = \mathbb{E}[\tilde{\fiii}(x;\paramii_0, \paramiii_0) \fii(x';\paramii_0)]\,.
    \end{align}
    As $\paramiii_0$ is drawn independently of $\paramii_0$, the conditional expectation can be written as 
    \begin{align}
        \mathbb{E}[\tilde{\fiii}(x;\paramii_0, \paramiii_0) \fii(x';\paramii_0) | \paramii_0] &= \fii(x';\paramii_0) \mathbb{E}[\tilde{\fiii}(x;\paramii_0, \paramiii_0)|\paramii_0] \\
        &= \fii(x';\paramii_0) \mathbb{E}[ \nabla_{\paramii^{\leq L-1}} \fii(x;\paramii_0)^\top \paramiii_0^{\leq L-1} | \paramii_0] \\ 
        &= \fii(x';\paramii_0) \cdot 0 \,,
    \end{align}
    and by the law of total expectation
    \begin{align}
        \mathbb{E}[\tilde{\fiii}(x;\paramii_0, \paramiii_0) \fii(x';\paramii_0)] &= \mathbb{E}_{\paramii_0} \bigl[ \mathbb{E}[\tilde{\fiii}(x;\paramii_0, \paramiii_0) \fii(x';\paramii_0) | \paramii_0]\bigr] \\
        &= 0 \,.
    \end{align}
    We conclude that the two GP $\tilde{\fiii}(x;\paramii_0, \paramiii_0) \sim \mathcal{GP}(0, \ntkk_{xx'}^{\leq L-1})$ and $\fii(x;\paramii_0) \sim \mathcal{GP}(0, \ntkk_{xx'}^{L})$ are mutually independent such that the initialization kernel $\nngpk^{\err^b}_{xx'}$ is given as 
    \begin{align}
        \nngpk^{\err^b}_{xx'} = \ntkk_{xx'} \,.
    \end{align}
    This is because $\ntkk_{xx'} = \ntkk^{L}_{xx'} + \ntkk^{\leq L-1}_{xx'}$ and $\nngpk^{\tilde{\fiii}}_{xx'} = \ntkk^{\leq L-1}_{xx'}$, $\nngpk^\fii_{xx'} = \ntkk^{L}_{xx'}$ are mutually independent. 
\end{proof}

\clearpage
\section{Additional Experimental Details} \label{app:experiments}

We report additional experimental details and evaluations. As outlined in the main text, we use two-layer fully connected neural networks with SiLU activations and NTK parametrization. All weights and biases are initialized as $\param \sim \mathcal{N}(0, I)$. We use an ensemble of 512 models and a single multiheaded RND network with 512 heads. A synthetic dataset is generated with $N=10$ train and $\tilde{N}=5000$ test samples from an isotropic Gaussian $x \sim \mathcal{N}(0,I_3)$. We label training samples with a synthetic target function 
\begin{equation}
    y(x) = x^0 + x^1 + x^2 - 2 \prod_{i=1}^3 x^i,
\end{equation}
where $x^i$ denotes the $i$-th component of vector $x$. All models are trained according to the algorithms outlined in the main text. For this, we use full-batch gradient descent with a learning rate of $0.1$ for all models. Fig. ~\ref{fig:additionalexperiments} shows additional results of the same experiment, in which we plot individual test-set ensemble variances against RND errors. As the network width increases, ensemble variances and self-predictive RND errors become more correlated and well-calibrated in scale. 

Code for full reproduction will be released upon publication. 

\begin{figure}[t]
    \centering
    \begin{subfigure}{0.45\textwidth}
        \centering
        \includegraphics[width=\columnwidth]{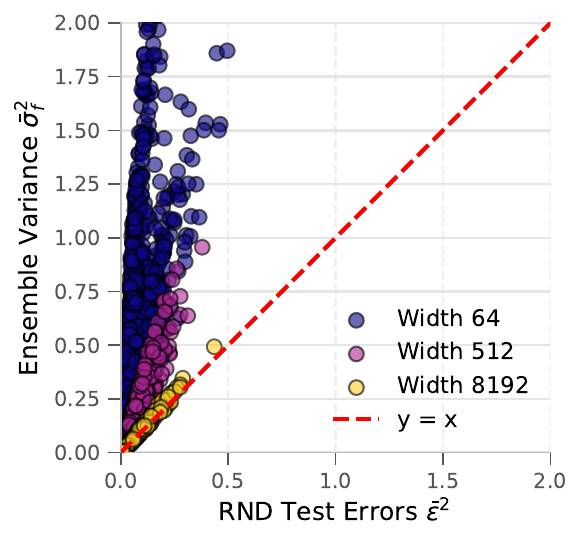}
    \end{subfigure}
    \hspace{6mm}
    \begin{subfigure}{0.43\textwidth}
        \centering
        \includegraphics[width=\columnwidth]{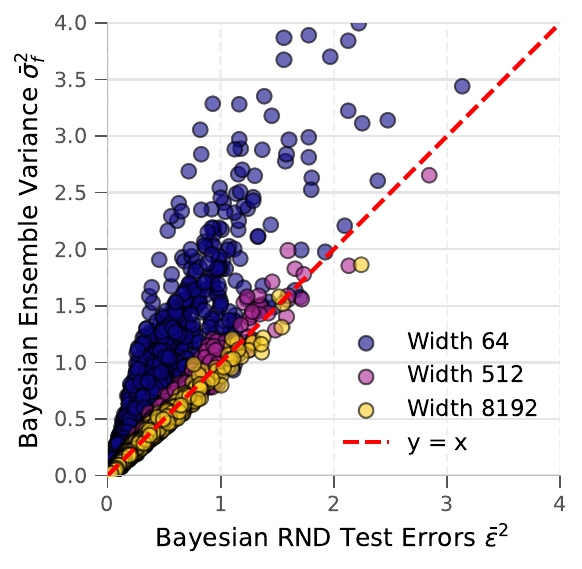}
    \end{subfigure}
    \begin{center}\footnotesize
    \vspace{-9.5mm}
    \hspace{-42mm}(a)\hspace{71mm}(b)
    \end{center}
    \caption{(a) Scatter plot of test-set errors between predictive variances of ensembles and self-predictive errors of RND. As width increases, errors become more correlated and correctly calibrated in scale. (b) Likewise, for \textit{Bayesian} ensembles and \textit{Bayesian} RND.}
    \label{fig:additionalexperiments}
\end{figure}

\end{document}